\documentclass[twoside]{article}

\usepackage[accepted]{aistats2025}
\usepackage[round]{natbib}
\usepackage{amsmath}
\usepackage{amssymb}
\usepackage{mathtools}
\usepackage{amsthm}
\usepackage{alias}

\theoremstyle{plain}
\newtheorem{theorem}{Theorem}[section]

\newtheorem{lemma}[theorem]{Lemma}
\newtheorem{corollary}[theorem]{Corollary}
\theoremstyle{definition}

\theoremstyle{remark}
\newtheorem{remark}[theorem]{Remark}

\begin{document}

%

%

\twocolumn[

\aistatstitle{Prior-Dependent Allocations for Bayesian Fixed-Budget Best-Arm Identification in Structured Bandits}

\aistatsauthor{Nicolas Nguyen$^*$ \And Imad Aouali$^*$ \And András György \And Claire Vernade}
\aistatsaddress{ University of Tübingen \And  Criteo AI Lab \\ CREST, ENSAE, IP Paris \And Google DeepMind \And University of Tübingen} ]

{\renewcommand{\thefootnote}{}\footnotetext{* Equal contribution.}}

\begin{abstract}
We study the problem of Bayesian fixed-budget best-arm identification (BAI) in structured bandits. We propose an algorithm that uses fixed allocations based on the prior information and the structure of the environment. We provide theoretical bounds on its performance across diverse models, including the first prior-dependent upper bounds for linear and hierarchical BAI. Our key contribution lies in introducing novel proof techniques that yield tighter bounds for multi-armed BAI compared to existing approaches. Our work provides new insights into Bayesian fixed-budget BAI in structured bandits, and extensive experiments demonstrate the consistent and robust performance of our method in practice across various settings. 
\end{abstract}

\section{INTRODUCTION}\label{sec:introduction}
Best arm identification (BAI) addresses the challenge of finding the optimal arm in a bandit environment \citep{lattimore2020bandit}, with wide-ranging applications in online advertising, drug discovery or hyperparameter tuning. BAI is commonly approached through two primary paradigms: \textit{fixed-confidence} and \textit{fixed-budget}. In the fixed-confidence setting \citep{even2006action,kaufmann2016complexity}, the objective is to find the optimal arm with a pre-specified confidence level. Conversely, fixed-budget BAI \citep[e.g., ][]{audibert2010best} involves identifying the optimal arm within a fixed number of observations. Within this fixed-budget context, two main metrics are used: the \emph{probability of error} (PoE) \citep{audibert2010best, karnin2013almost, carpentier2016tight}, which is the likelihood of incorrectly identifying the optimal arm, and the \emph{simple regret} \citep{bubeck2009pure, russo2016simple,komiyama2023rate}, which corresponds to the expected performance disparity between the chosen and the optimal arm. We study PoE minimization in fixed-budget BAI in the Bayesian setting \citep{atsidakou2022bayesian}, a problem closer to statistical hypothesis testing \citep{bernardo2002bayesian} than to regret minimization \citep{lattimore2020bandit}, and often viewed as a more difficult question in general \citep{carpentier2016tight,degenne2023existence}. 

\textbf{The question of adaptivity in fixed-budget BAI.} Existing algorithms for PoE minimization in fixed-budget BAI are largely frequentist and often employ elimination strategies, such as Sequential Halving (SH) \citep{karnin2013almost}, that can be derived for linear models using optimal designs \citep{azizi2021fixed,hoffman2014correlation,katz2020empirical,yang2022minimax}. 
These algorithms are called \emph{adaptive} in the sense that the sampling strategy progressively allocates more budget to higher-valued arms. That being said, BAI is a hard problem whose complexity remains hard to fully understand \citep{degenne2023existence}, and for which simple baselines can be deceivingly strong on some problem instances. \citet{wang2024universally} showed that for two-armed Bernoulli bandits, there is no algorithm that is strictly better than the uniform sampling algorithm in all instances when comparing the frequentist PoE.  
In fact, \citet{degenne2023existence} established that there is no superior adaptive algorithms in several BAI problems, including Gaussian BAI.
This questions the role of adaptivity for fixed-budget BAI problems \citep{qin2022open}.

In the Bayesian setting, however, when prior knowledge is available to the learner, the competitiveness of the uniform strategy must be reassessed. As a motivating example, consider a scenario with 3 arms, 
where prior information strongly suggests that either of the first two arms is more likely to be optimal than the third. How much budget should then be allocated to the seemingly suboptimal third arm? Moreover, if there is greater confidence in the first arm compared to the second, 
what is the optimal budget distribution between them? 
In this paper, we address the following general question: \emph{How to choose a prior-dependent fixed allocation of the budget to guarantee a low expected PoE?}

\textbf{Prior works on Bayesian BAI. }
Bayesian approaches have focused on the minimization of the simple regret \citep{azizi2023meta,komiyama2023rate}, or were studied under a frequentist lens \citep{hoffman2014correlation,yang2024improving}, which do not capture the advantages of knowing informative priors. Only recently, \citet{atsidakou2022bayesian} introduced a Bayesian version of Sequential Halving for BAI in the $K$-armed Gaussian multi-armed bandit (MAB) setting. They incorporate prior knowledge in two ways: the allocations in each elimination phase prioritize arms with higher observation noise variance, and the final decision is based on the posterior mean which is a function of the prior and of the observations. 
However, their analysis directly relies on averaging frequentist upper bounds over the prior and require assumptions on the prior distribution; unfortunately, for technical reasons,\footnote{We provide a detailed discussion in \cref{sec:app_related_work,sec:analysis}.} it does not generalize to more informative priors and structured bandits. In particular, to the best of our knowledge, it cannot be extended to linear bandits through optimal design. 
Our work provides a novel algorithm that naturally applies to structured bandits and gives a general answer to Bayesian fixed-budget BAI with non-adaptive allocation strategies.

\textbf{Contributions.} \textbf{1)} We present and analyze a fixed-budget BAI algorithm, which we call Prior-Informed BAI (\AdaBAI), that leverages prior information for efficient exploration. Our main contributions include establishing prior-dependent upper bounds on its expected PoE in multi-armed, linear, and hierarchical bandits. Specifically, in the MAB setting, our upper bound is tighter and valid under milder assumptions on the prior than that of \citet{atsidakou2022bayesian}. \textbf{2)} The proof techniques developed for \AdaBAI provide a fully Bayesian perspective, significantly diverging from existing proofs that rely on frequentist techniques. This gives a more comprehensive framework for understanding and analyzing Bayesian BAI algorithms, while also enabling us to relax previously held assumptions. \textbf{3)} As a result, unlike earlier approaches, our algorithms and proofs are naturally applicable to structured problems, such as linear and hierarchical bandits, leading to the first Bayesian BAI algorithm with a prior-dependent PoE bound in these settings. \textbf{4)}  We empirically evaluate \AdaBAI's variants in numerous setups. Our experiments on synthetic and real-world data show the generality of \AdaBAI and highlight its good performance.

\section{A NEW ALGORITHM CONCEPT}
\label{sec:background}
\textbf{Notation.} For any positive integer $K$, let $\Delta_K$ be the $K$-simplex and $\Delta_K^+ = \{\omega \in \Delta_K\,; \, \omega_i > 0 \,, \,  \forall i \in [K]\}$, where $[K]=\{1,2,\ldots,K\}$. For any positive-definite matrix $A \in \real^{d\times d}$ and vector $v \in \real^d$, we define $\normw{v}{A} = \sqrt{v^\top A v}$. Also, $\lambda_1(A)$ and $\lambda_d(A)$ denote the maximum and minimum eigenvalues of $A$, respectively. We denote by $e_i$ the $i$-th vector of the canonical basis, and by $\mathrm{KL}$ the Kullback-Leibler divergence.

\textbf{Background.} We consider scenarios involving $K$ arms. In each round $t \in [n]$, the agent selects an arm $A_t \in [K]$, and then receives a stochastic reward $Y_t \sim P(\cdot; \theta, A_t)$, where $\theta$ is an \emph{unknown} parameter vector and $P(\cdot; \theta, i)$ is the \emph{known} reward distribution of arm $i$, given $\theta$. We denote by $r(i; \theta)=\E{Y \sim P(\cdot; \theta, i)}{Y}$ the mean reward of arm $i$, given $\theta$.
We adopt the Bayesian view where $\theta$ is assumed to be sampled from a \emph{known} prior distribution $P_0$. Given a bandit instance characterized by $\theta \sim P_0$, the goal is to find the (unique) optimal arm $i_*(\theta)=\argmax_{i \in [K]} r(i; \theta)$ by interacting with the bandit instance for a fixed-budget of $n$ rounds.  
These interactions are summarized by \emph{the history} $H_n=\{(A_t, Y_t)\}_{t=1}^n$, and we denote $J_n \in [K]$ the arm selected by the agent after $n$ rounds. In this Bayesian setting, \citet{atsidakou2022bayesian} study the Bayesian risk, which they call \emph{expected PoE}:
\begin{align*}
    \mathcal{P}_n =\E{}{\mathbb{P}\big(J_n \neq i_*(\theta) \mid \theta \big)}\,,
\end{align*} 
that corresponds to the average PoE over instances sampled from the prior, $\theta \sim P_0$. This is different from the frequentist counterpart where the performance is assessed for a single $\theta$.

\subsection{Prior-Informed BAI}
\vspace{-0.1em}
We consider the following algorithm: \\[-0.8em]
\begin{minipage}{0.5\textwidth}
\begin{algorithm}[H]
\caption{Prior-Informed BAI: $\AdaBAI(\texttt{Alloc})$}
\label{alg:ada_bai_fixed}
\textbf{Input:} Budget $n$, prior $P_0$, function \texttt{Alloc} (e.g. Sec~\ref{subsec:allocations})
Compute allocations $\omega = \texttt{Alloc}(n, P_0)$.\\
\For{$i=1, \dots, K$}{Get $n_{i} = \lfloor \omega_i n \rfloor $ samples of arm $i$\\
Compute mean posterior reward $\E{}{r(i; \theta) \mid H_{n}}$}
Set $J_n = \argmax_{i \in [K]}\E{}{r(i; \theta) \mid H_{n}}$.
\end{algorithm}
\end{minipage}

\AdaBAI (\cref{alg:ada_bai_fixed}) takes as input
a prior distribution $P_0$, a budget $n$ and an allocation function \texttt{Alloc}. Then, \AdaBAI uses the function \texttt{Alloc} to compute a vector of prior-dependent allocation weights $\omega = (\omega_i)_{i \in [K]} \in \Delta_K$. This vector is central to the good performance of our method; choices how to compute it are discussed in details in the next section. After that, \AdaBAI  proceeds in a straightforward manner: it collects $n_{i} = \lfloor \omega_i n \rfloor$ samples for each arm $i$, and finally returns the arm with the highest mean posterior reward, defined as 
\begin{align}\label{eq:decision_rule}
\textstyle
    J_n = \argmax_{i \in [K]}\,\E{}{r(i; \theta) \mid H_{n}}\,.
\end{align}
This posterior expectation can be computed in closed-form in Gaussian bandits with Gaussian priors. Note that a fully Bayesian algorithm would set $J_n = \argmin_{j\in[K]}\mathbb{P}(j\neq i_*(\theta) \,|\, H_n)$; however, this cannot be computed explicitly even in the simple Gaussian setting, and we instead opt for a decision rule that is deterministic conditionally on observations, as in top-two based algorithms in the fixed-confidence setting \citep{jourdan2022top}.

Decision rule \eqref{eq:decision_rule} 
naturally generalizes to structured bandit settings: the structure is a direct prior information that is taken into account in the computation of the posterior updates and the allocation vector. Similarly, and despite additional technicalities, our novel proof scheme (see \cref{sec:analysis}) is preserved across all settings we consider. 

\textbf{Analysis overview.} For any prior-dependent allocation vector $\omega \in \Delta^+_K$ (computed by \texttt{Alloc}), we derive an upper bound on the PoE in multi-armed, linear and hierarchical bandits (see below \cref{th:upper bound,th:linear_model_upper_bound,th:upper_bound_hierarchical_model}) of the form 
\begin{equation}
\label{eq:general-analysis}
  \mathcal{P}_n \leq C(P_0, \omega, n),  
\end{equation}
where the bound $C$ depends on the prior $P_0$, the budget $n$ and the allocation weights $\omega$. The bound is valid for any $\omega \in \Delta^+_K$, but its guarantee, as well as the performance of the algorithm, weakens for worse choices of $\omega$. To alleviate this, a key aspect of our work is to derive good (prior-dependent) allocation strategies.


\section{ALLOCATIONS STRATEGIES}\label{subsec:allocations}
In this section, we propose three approaches to derive allocation strategies to be used in \AdaBAI, each leveraging the prior in different ways, and we show in the next section how they can be implemented in structured bandit models. 

\textbf{Allocation by optimization.} Since the bound \eqref{eq:general-analysis} is valid for any $\omega \in \Delta^+_K$, we can define the \emph{optimized allocation weights} as
\begin{talign}\label{eq:optimized_weights}
&\text{Opt}(n, P_0) = \argmin_{w \in \Delta_K^+} C(P_0, \omega, n) \,.
\end{talign}
We call the resulting algorithm $\AdaBAI(\text{Opt})$, where we set the allocation function $\texttt{Alloc} = \text{Opt}$. 
In general, \eqref{eq:optimized_weights} is non-convex and may require advanced numerical methods \citep[e.g., L-BFGS-B of ][]{2020SciPy-NMeth} to find a suitable optimum. 
Fortunately, this optimization is performed only once (offline) before interacting with the environment. 

In our experiments, as a heuristic, we define a mixture of optimized weights: \(\alpha \wOpt_i + (1 - \alpha)\frac{\mu_{0,i}\sigma_{0,i}}{\sum_{k \in [K]}\mu_{0,k}\sigma_{0,k}}\) where $\alpha \in (0, 1)$ whose performance is still covered by our upper bound. We tested various \(\alpha\) values (\cref{subsec:app_experiments_hyperparameters}) and found the best performance at \(\alpha \approx 0.5\). For notational simplicity, we still refer to this allocation weight (with \(\alpha=0.5\)) as \(\wOpt\).

\textbf{Allocation by G-optimal design.} 
Optimal experimental design \citep[Chapter 21]{lattimore2020bandit} is an optimization problem that returns an allocation over arms that balances exploration and exploitation, usually used in linear bandit algorithms with fixed-design \citep{abbasi-yadkori11improved}. We notice that this idea can also be leveraged to obtain allocations for \emph{Bayesian} BAI.
Finding an (approximate) Bayesian G-optimal design \citep[Chapter 4]{lopez2023optimal} is equivalent to maximizing the log-determinant of the regularized information matrix defined as 
\begin{align}
    \label{eq:optimal_design} 
    \!\!\text{G-opt}(n, P_0) = 
    \argmax_{\omega \in \Delta_K^+}\, \log \det\left(\frac{n}{\sigma^2}\Sigma_\omega + \Sigma_0^{-1}\!\right)\!,
\end{align}
where $(x_1,\dots x_k)$ are the action vectors (unit vectors in the MAB setting) and $\Sigma_\omega = \sum_{k \in [K]}\!\omega_k x_k x_k^\top$.
This leads to budget allocations that minimize the worst-case posterior variance in all directions.

Note that both allocation weights $\text{Opt}(n, P_0)$ \eqref{eq:optimized_weights} and $\text{G-opt}(n, P_0)$ \eqref{eq:optimal_design} are prior-dependent and fixed in advance before interacting with the environment (hence they are independent of the actual instantiation $\theta$ of the environment). Therefore, both enjoy the theoretical guarantees we derive in \cref{subsec:bound_MAB,subsec:bound_structured} that hold for any allocation weights fixed in advance.

\textbf{Allocation by warm-up.} Finally, we introduce an improper `adaptive' strategy that calls a possibly prior-dependent \emph{warm-up} policy
\(\pi_{\rm{w}}\) to interact with the bandit environment for \(n_{\rm{w}}\) rounds (the warm-up phase). This no longer corresponds exactly to the \AdaBAI algorithm as this \texttt{Alloc} strategy also interacts with the environment and outputs instance-dependent weights  \(\omega_{\pi_{\rm{w}}}\), thus our theoretical guarantees do not apply to this framework. 
The warm-up policy can be any decision-making strategy that (preferably) also uses the prior. Inspired by its strong performance in BAI \citep{lee2024thompson}, we use Thompson sampling for \(\pi_{\rm{w}}\). The allocation weights are then set proportional to the number of pulls for each arm during the warm-up phase: \(\wTS_i = \frac{1}{n_{\rm{w}}}\sum_{t \in [n_{\rm{w}}]}\I{A_t = i}\) for all \(i \in [K]\). These strategies are further discussed in \cref{subsec:app_differences_opt_and_TS}.

\vspace{-0.2cm}
\section{BANDIT SETTINGS}
\vspace{-0.2cm}
In this section, we present how to concretely implement our algorithm \AdaBAI for several major bandit settings using the three types of allocations discussed above, and prove the corresponding theoretical guarantees.
\subsection{Multi-Armed Bandits}
\label{subsec:bound_MAB}
In this setting, each component of $\theta = (\theta_i)_{i \in [K]}$ is sampled independently from the prior distribution. We focus on the Gaussian case where $\theta_i \sim \cN(\mu_{0, i}, \sigma_{0, i}^2)$, with $\mu_{0, i}$ and $\sigma_{0, i}^2$ being the \emph{known} prior reward mean and variance for arm $i$. Then, given $\theta$, the reward distribution of arm $i$ is $\cN(\theta_{i}, \sigma^2)$ where $\sigma^2$ is the (known) observation noise variance,\footnote{The noise variance $\sigma^2$ could be arm-dependent and we can provide a similar analysis, but we choose to keep notation simple.}
\vspace{-0.1cm}
\begin{talign}\label{eq:bayes_elim_model_gaussian}
   \theta_i &\sim \cN(\mu_{0, i}, \sigma_{0, i}^2)  & \forall i \in [K]\nonumber\\
    Y_t \mid \theta, A_t &\sim \cN(\theta_{A_t}, \sigma^2)  & \forall t \in [n]\,.
\end{talign}
Under \eqref{eq:bayes_elim_model_gaussian}, the posterior distribution of $\theta_i$ given $H_n$ is a Gaussian $\cN(\hat\mu_{{n}, i}, \sigma_{n, i}^2)$ \citep{Bishop2006} with
\begin{talign}\label{eq:bayes_elim_posterior}
     &\hat\sigma_{n, i}^{-2} = \sigma_{0, i}^{-2}  + n_i\sigma^{-2}\,, &
      \hat\mu_{n, i} = \hat{\sigma}_{n, i}^{2} 
      \left( \tfrac{\mu_{0, i}}{\sigma_{0, i}^{2}} +\tfrac{B_{n, i}}{\sigma^{2}}\right),
\end{talign}
where, defining 
$\mathcal{T}_i = \{t\in[n],\, A_t = i\}$ to be the set of rounds when arm $i$ is chosen, 
$n_i = |\mathcal{T}_i |$ is the number of times arm $i$ is chosen and $B_{n, i} = \sum_{t \in \mathcal{T}_i} Y_{t}$ is the sum of rewards of arm $i$. The \emph{mean posterior reward} in this case is $\E{}{r(i;\theta)\mid H_n} = \hat\mu_{n,i}$.

\textbf{PoE bound for MAB.} \cref{th:upper bound} presents an upper bound on the expected PoE of $\AdaBAI$ for MAB \eqref{eq:bayes_elim_model_gaussian} for a \emph{fixed and possibly prior-dependent} choice of positive allocation weights. The general proof scheme is presented in \cref{sec:analysis}, and the full proof is provided in \cref{subsec:app_proof_MAB}.
\begin{theorem}[Upper bound for MAB]\label{th:upper bound} For all $\omega \in \Delta^+_K$, the expected PoE of \emph{$\AdaBAI$} that uses allocation $\omega$ under the MAB problem \eqref{eq:bayes_elim_model_gaussian} is upper bounded as
\begin{talign*}
  \mathcal{P}_n &\leq \doublesum \frac{1}{\sqrt{1 + n \phi_{i,j}(\omega)}} \exp\left( -\frac{(\mu_{0, i} - \mu_{0, j})^2}{2(\sigma^2_{0, i} + \sigma^2_{0, j})}\right) \,,
\end{talign*}
where
\begin{align*}
\phi_{i, j}(\omega) = 
    \frac{\sigma^4_{0, i}\omega_i\left(\frac{\sigma^2}{n} + \omega_{j}\sigma^2_{0, j}\right) + \sigma^4_{0, j}\omega_{j}\left(\frac{\sigma^2}{n} + \omega_{i}\sigma^2_{0, i}\right)}{\sigma^2 \sigma^2_{0, i}\left(\frac{\sigma^2}{n} + \omega_{j}\sigma^2_{0, j}\right) + \sigma^2 \sigma^2_{0, j}\left(\frac{\sigma^2}{n} + \omega_{i}\sigma^2_{0, i}\right)}\,.
\end{align*}
In particular, $\phi_{i,j}(\omega)=\Omega(1)$ depends on the prior parameters and allocation weights, and  $\mathcal{P}_n = \mathcal{O}(1/\sqrt{n})$.
\end{theorem}
As a sanity check, note that if the prior is informative, either with small prior variances $\sigma_{0, i}^2 \to 0$ or large prior gaps $|\mu_{0, i} - \mu_{0, j}| \to \infty$, then $\mathcal{P}_n \to 0$ for any fixed allocation weights $\omega \in \Delta_K^+$. 

\textbf{Tightness of our bound.}
First we compare our upper bound to the 
lower bound of \citet{atsidakou2022bayesian} for the 2-armed Gaussian setting with $\sigma^2_{0, 1}=\sigma^2_{0, 2}=\sigma_0^2$:
\begin{talign*}
    \frac{e^{-\frac{(\mu_{0, 1}-\mu_{0, 2})^2}{4\sigma^2_0}}}{\sqrt{1 + \frac{n\sigma^2_0}{2\sigma^2}}} \overset{\text{Th.~\ref{th:upper bound}}}{\geq} \mathcal{P}_n \overset{\text{LB}}{\geq}  \frac{e^{-\frac{(\mu_{0, 1}-\mu_{0, 2})^2}{4\sigma^2_0}}}{2e\sqrt{1 + \frac{2n(8\log(2n)+1)\sigma^2_0}{\sigma^2}}}
\end{talign*}
Comparing the formulas above, we can see that in 2-arm problems we achieve optimal guarantees up to log factors and multiplicative universal constants. 

To further test the tightness of our bound, we numerically compare it with that of \BE \citep{atsidakou2022bayesian} using the 3-armed bandit example in \cref{sec:introduction}, where one arm is a priori suboptimal, while one of the other two is a priori optimal, reflected by prior means \(\mu_0 = (1, 1.9, 2)\). We consider two scenarios: one with homogeneous variances \(\sigma_{0, 1}= \sigma_{0, 2}= \sigma_{0, 3}= 0.3\) and another with heterogeneous variances \((\sigma_{0, 1}, \sigma_{0, 2}, \sigma_{0, 3}) = (0.1, 0.5, 0.5)\). Since \BE's bound does not handle heterogeneous prior variances, we use an average prior variance \(\sigma^2_{0} = \frac{1}{K}\sum_{i \in [K]}\sigma^2_{0, i}\) for comparison. 
\AdaBAI is instantiated with the following allocation strategies: Opt, using\eqref{eq:optimized_weights} with \cref{th:upper bound}; Uniform, with $\omega_i=1/K$ for all $i$; Random, where $\omega_i$ are generated uniformly at random from $[0,1]$ and then normalized to sum up to 1; and Heuristic, which gives more weight to arms with higher prior means and variances as \(\omega_i \propto \mu_{0, i}\sigma_{0, i}\).

As shown in \cref{fig:upper_bound}, our bound is tighter than the upper bound of \citet{atsidakou2022bayesian} for any allocations and in both variance settings. This gain is due to two main reasons: our bound is tighter by a factor \(\log_2^{3/2}(K)\) because our algorithm does not rely on elimination phases. Moreover, our proof technique yields terms that explicitly depend on the structure of the prior, which provides deeper insights beyond the $\mathcal{O}(1/\sqrt{n})$ rate typically obtained by averaging the frequentist bound of $\mathcal{O}(e^{-n/f(\theta)})$ \citep{audibert2010best,carpentier2016tight} over a Gaussian prior, though it does not necessarily result in faster rates.

\begin{figure}
    \centering
    \includegraphics[width=0.8\linewidth]{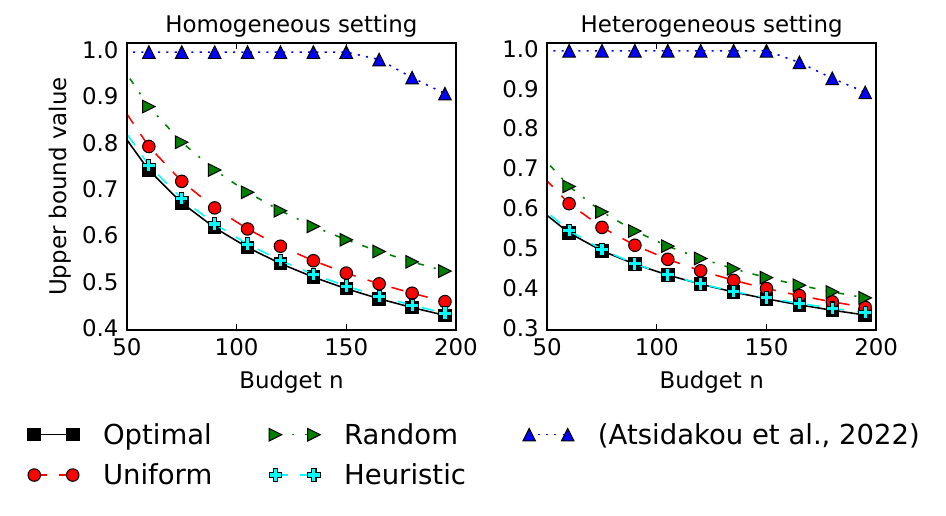}
\vspace{-2mm}
    \caption{Bound of \AdaBAI with different weights compared to \BE \citep{atsidakou2022bayesian}.}
    \label{fig:upper_bound}
\end{figure}

\textbf{Allocation by optimization: a simple example in the MAB setting.} To provide intuition, we present the explicit solution of \eqref{eq:optimized_weights} with \cref{th:upper bound} for the MAB setting with \(K=2\): denoting $\wOpt = \text{Opt}(n, P_0)$,
\begin{talign*}
    \wOpt_1 = \Pi_{[0, 1]}\left(\frac{1}{2} - \frac{(\sigma^2_{0, 2} - \sigma^2_{0, 1})\sigma^2}{2n\sigma^2_{0, 1}\sigma^2_{0, 2}}\right)\,, \; \text{and} \; \wOpt_2 = 1-\wOpt_1,
\end{talign*}
where \(\Pi_{[a, b]}(\cdot)\) denotes truncation into \([a,b]\). 
First, if the prior variances are equal (\(\sigma^2_{0, 1} = \sigma^2_{0, 2}\)), allocating an equal number of samples to each arm is optimal. Conversely, if \(\sigma^2_{0, 1} \ll \sigma^2_{0, 2}\) and the budget \(n\) is small, we get \(\wOpt_2 \gg \wOpt_1\), allocating most of the budget to the arm with higher prior variance (lower initial confidence). This is only relevant for small budgets since when \(n \to +\infty\), the optimal strategy converges to uniform allocation as the influence of the prior vanishes asymptotically. On the other hand, a saturation phenomenon happens for `too small' budgets: if $n<2\sigma^2 \left| \frac{\sigma_{0,2}^2}{\sigma_{0,1}^2} - \frac{\sigma_{0,1}^2}{\sigma_{0,2}^2}\right|$, the weight of the arm with larger prior variance is equal to 1 due to the truncation. Finally, this allocation strategy does not depend on the prior gap \(\Delta_0\), which is expected as identifying the optimal arm is equivalent to identifying the worst arm in the case of two arms. This equivalence does not hold when \(K>2\).


\textbf{Allocation by optimal design.} G-optimal design can be applied for MAB, by using $x_i = e_i$ for any $i \in [K]$ and $\Sigma_0 = {\rm{diag}}(\sigma_{0, i}^2)_{i \in [K]}$. 

\textbf{Allocation by warm-up.} Thompson Sampling is well-understood for MAB \citep{russo18tutorial} even beyond Gaussian priors so this allocation strategy is fairly versatil, though not covered by our theory since the resulting weights are instance-dependent.

\subsection{Linear Bandits}
\label{subsec:linear_bandit}
\label{subsec:bound_structured}
In linear bandits \citep{abbasi-yadkori11improved}, arms share a common low-dimensional representation through the parameter $\theta \in \real^d$. We denote $\cX = \{x_1,\dots x_K\}$ the set of arms where each $x_i \in \real^d$. 
In the Gaussian case, 
the reward distribution is parametrized by $\theta$,
\begin{talign}\label{eq:linear_gaussian}
   \theta &\sim \cN(\mu_{0}, \Sigma_0)  &\nonumber\\
    Y_t \mid \theta, A_t &\sim \cN(\theta^\top x_{A_t}, \sigma^2)  & \forall t \in [n]\,,
\end{talign}
for some prior mean $\mu_0 \in \real^d$ and variance $\Sigma_0 \in \real^{d\times d}$.
Similarly to \eqref{eq:bayes_elim_model_gaussian}, this model offers closed-form formulas \citep{agrawal13thompson}, where the posterior of $\theta$ given the history $H_n$ containing $n_i$ samples from arm $i$ is a Gaussian $\cN(\hat\mu_{{n}}, \Sigma_{n})$:
\begin{talign}\label{eq:linear_gaussian_posteriors}
     &\hat\Sigma_{n}^{-1} = \Sigma_0^{-1}  + \sigma^{-2} \sum_{i \in [K]}n_i x_{i} x_{i}^\top\nonumber\\
     &\hat\mu_{n} = \hat\Sigma_n \big( \Sigma_0^{-1}\mu_0 + \sigma^{-2} B_n \big)\,,
\end{talign}
where $B_{n} = \sum_{t \in [n]} Y_t x_{A_t}$. The mean posterior reward of arm $i$ is given by $\E{}{r(i;\theta)\mid H_n} = \hat\mu_{n}^\top x_{i}$. Note that the MAB \eqref{eq:bayes_elim_model_gaussian} can be recovered from \eqref{eq:linear_gaussian} when $x_i = e_i \in \real^K$ and $\Sigma_0 = {\rm{diag}}(\sigma_{0, i}^2)_{i \in [K]}$. 

\textbf{PoE bound for linear bandits.} Importantly, our error analysis extends to structured bandit problems. We begin with linear bandits in the next theorem and discuss the specific case of hierarchical models further below.


\begin{theorem}[Upper bound for linear bandits]
\label{th:linear_model_upper_bound}
Assume that $x_i \neq x_j$ for any $i \neq j$, and that $\mathcal{X}$ spans $\mathcal{R}^d$.
Then, for all $\omega \in \Delta^+_K$, the expected PoE of \emph{$\AdaBAI$} using allocation $\omega$ under the linear bandit problem \eqref{eq:linear_gaussian} is upper bounded as
\begin{talign*}
  \mathcal{P}_n &\leq \doublesum \frac{1}{\sqrt{1 + \frac{n c_{i,j}(\omega) }{ \lVert x_i - x_j\rVert^2_{\tilde\Sigma_n} }}}\exp\left(-\frac{( \mu_0^\top x_i - \mu_0^\top x_j)^2}{2\lVert x_{i} - x_{j}\rVert^2_{\Sigma_0}}\right)\,,
\end{talign*}
where $c_{i,j}(\omega)$ is given in \eqref{eq:linear_c_i_j}, and $\tilde\Sigma_n = (\frac{1}{n} \Sigma_0^{-1} + \sigma^{-2} \sum_{i \in [K]} w_i x_ix_i^\top)^{-1}$ is deterministic. 
\end{theorem}
In particular, $\mathcal{P}_n = \mathcal{O}(1/\sqrt{n})$ and we recover \cref{th:upper bound} when $x_i = e_i \in \real^K$ and $\Sigma_0 = {\rm{diag}}(\sigma_{0, i}^2)_{i \in [K]}$. Full expressions and proofs are given in \cref{subsec:app_proof_linear_bandit}.
Also, this bound captures the benefit of using informative priors since $\mathcal{P}_n \to 0$ when the prior variances are small, that is, $\Sigma_0\to 0_{L\times L}$, or when the prior gaps are large, in the sense that $|\mu_0^\top x_i - \mu_0^\top x_j| \to \infty$.

\textbf{Allocation by optimization. } Similarly to the MAB setting, the bound in \cref{th:linear_model_upper_bound} can be numerically optimized. This is a non-convex problem so there is no guarantee to reach a global optimum,\footnote{Our code is provided for reproducibility purposes.} and, in practice, initialization plays an important role. 

\textbf{Allocation by warm-up. } 
We use Linear Thompson Sampling \citep[e.g.,][]{abeille17linear} with the given informed prior $\mathcal{N}(\mu_0,\Sigma_0)$. 

\textbf{G-optimal design allocations. }
In this setting, the most appropriate allocation strategy is by Bayesian G-optimal design \eqref{eq:optimal_design}.
We specifically compute our upper bound with this choice below.
\begin{corollary}[Upper bound of $\AdaBAI(\text{G-opt})$]
\label{cor:linear_with_optimal_design}
Under the assumptions of \cref{th:linear_model_upper_bound},
    the expected PoE of \emph{$\AdaBAI(\text{G-opt})$} in problem \eqref{eq:linear_gaussian} with a diagonal covariance matrix $\Sigma_0$ is upper bounded as
    \begin{talign*}
        \mathcal{P}_n \leq \sum_{\substack{i,j \in [K] \\ i \neq j}} \frac{1}{\sqrt{1 + \frac{n}{2d\sigma^2}c_{i,j}}}\exp\left( -\frac{(\mu_0^\top x_i - \mu_0^\top x_j)^2}{ 2\lVert x_i - x_j \rVert^2_{\Sigma_0} }\right)\,,
    \end{talign*}
   where $c_{i,j}=\Omega(1)$. The full expressions and proofs are given in \cref{subsec:app_proof_linear_bandit}. 
\end{corollary}

\subsection{Hierarchical Bandits}\label{subsec:hierarchical_bandit}
Hierarchical (or \emph{mixed-effect}) models are rich models commonly used in practice where actions are typically structured into clusters, or \emph{effects} \citep{aouali2022mixed,Bishop2006,hong22hierarchical,wainwright2008graphical}. The correlations are then captured as 
\begin{talign}\label{eq:model_gaussian}
    \mu &\sim \cN(\nu, \Sigma)\nonumber\\
    \theta_i &\sim \cN(b_i^\top \mu, \sigma_{0, i}^2)  & \forall i \in [K]\nonumber\\
    Y_t \mid \mu, \theta, A_t &\sim \cN(\theta_{A_t}, \sigma^2)  & \forall t \in [n]\,
\end{talign} 
where $\mu = (\mu_\ell)_{\ell \in [L]}$ is an unknown latent parameter composed of $L$ effects $\mu_\ell$ and it is sampled from a multivariate Gaussian  with mean $\nu \in \real^L$ and covariance $\Sigma \in \real^{L \times L}$. Then, given $\mu$, the mean rewards $\theta_i$ are independently sampled as $\theta_i \sim \cN(b_i^\top \mu, \sigma_{0, i}^2)$, where $(b_i)_{i \in [K]}$ represent \emph{known mixing weights}. In particular, $b_i^\top \mu$ creates a linear mixture of the $L$ effects, with $b_{i, \ell}$ indicating a known score that quantifies the association between arm $i$ and the effect $\ell$. Concrete examples of $b_{i, \ell}$ are provided in \cref{subsec:app_motivating_examples}. Note that arm correlations arise because $\theta_i$ are derived from the same effect parameter $\mu$. Finally, given $\mu$ and $\theta$, the reward distribution of arm $i$ is similar to the MAB \eqref{eq:bayes_elim_model_gaussian} and only depends on $\theta_i$ as $\cN(\theta_{i}, \sigma^2)$.

Similarly to the linear case, an explicit marginal posterior distribution $\mathbb{P}(\theta_i \mid H_n)= \cN(\hat{\mu}_{n, i}, \hat{\sigma}_{n, i}^2)$ can be computed for closed-form posterior mean and variance parameters (see \cref{subsec:app_proof_technical}).

\paragraph{PoE bound for hierarchical bandits.} We derive a bound on the PoE that takes advantage of the hierarchical structure \eqref{eq:model_gaussian}. To the best of our knowledge, these are the first \emph{prior-dependent} bounds for fixed-budget Bayesian BAI in these settings.
\begin{theorem}[Upper bound for hierarchical bandits]
\label{th:upper_bound_hierarchical_model}
    For all $\omega \in \Delta^+_K$, the expected PoE of \emph{$\AdaBAI$} using allocaiton $\omega$ under the hierarchical bandit problem \eqref{eq:model_gaussian} is upper bounded as
    \begin{talign*}
    \mathcal{P}_n \leq \doublesum \frac{1}{\sqrt{1 + \frac{ c_{i,j}(\omega) }{\hat\sigma_{n,i}+\hat\sigma_{n,j}}}}\exp\left(\frac{-(\nu^\top b_i - \nu^\top b_j)^2}{2( \normw{b_i - b_j}{\Sigma}^2 + \sigma^2_{0, i}+\sigma^2_{0, j})}\right),
    \end{talign*}
    where $c_{i,j}(\omega)$ is given in \eqref{eq:hierarchical_c_i_j} and $\hat\sigma_{n,i}$ in \eqref{eq:posterior}. $c_{i,j} = \Omega(1)$ and $\hat\sigma^2_{n, i} = \Omega(1/n)$, and they depend on both the prior parameters and allocation weights. In particular, $\mathcal{P}_n = \mathcal{O}(1/\sqrt{n})$. Full expressions and proofs are given in \cref{subsec:app_proof_hierarchical}.
\end{theorem}
The term $\normw{b_i - b_j}{\Sigma}^2 + \sigma^2_{0, i}+\sigma^2_{0, j}$ accounts for the prior uncertainty of both arms and effects. If the effects are deterministic ($\Sigma = 0_{L\times L}$) then our bound recovers the upper bound of MAB with prior mean $\mu_{0, i} = \nu^\top b_i$. On the other hand, if the arms are deterministic given the effects ($\sigma^2_{0, i}= 0$), the bound only depends on the effect covariance. Finally, if the prior is informative by its gap ($|\nu^\top b_i - \nu^\top b_j| \to \infty\,$) or by its variance ($\Sigma \to 0_{L\times L}$ and $\sigma^2_{0, i} \to 0$), then $\mathcal{P}_n \to 0$.

\textbf{Allocation by optimization.} Again, numerical optimization can be leveraged for the upper bound in \cref{th:upper_bound_hierarchical_model}, but in our experiments (see \cref{sec:experiments}) we see that this method reaches its limits due to the complexity of the bound to be optimized.

\textbf{Allocation by optimal design.} \eqref{eq:model_gaussian} is a special case of a linear bandit \eqref{eq:linear_gaussian}, by realizing that $\theta_i = b_i^\top \mu + \eta_i$ where $\eta_i \sim \cN(0,\sigma_{0,i}^2)$ and the $\eta_i$ are independent of $\mu$. Hence, \eqref{eq:model_gaussian} can be rewritten by replacing $\nu$ with $\bnu \in \real^{L+K}$ defined as $\bnu^\top=(\nu^\top,0,\ldots,0)$ and $\Sigma$ with a block-diagonal matrix $\bSigma= \langle \Sigma, \sigma_{0,1}^2,\ldots,\sigma_{0,K}^2 \rangle$ with actions $\bb_i \in \real^{L+K}$ defined as $\bb_i^\top= (b_i^\top,e_i^\top)$, leading to a linear bandit
\begin{talign}\label{eq:first_lb}
    \bmu &\sim \cN(\bnu, \bSigma)\nonumber\\
    Y_t \mid \bmu, A_t &\sim \cN(\bb^\top_{A_t} \bmu, \sigma^2)  & \forall t \in [n]\,.
\end{talign}
The only practical benefit of this point of view is to allow us to compute G-optimal design allocations \eqref{eq:optimal_design} and inherit the theoretical guaranties from the linear case. In general, adhering to the initial formulation in \eqref{eq:model_gaussian} and the subsequent derivations in \eqref{eq:posterior} is more computationally efficient and this motivates the concept of hierarchical bandits in the first place. Additional discussions can be found in \cref{subsec:app_remarks_hierarchical_models}.

\textbf{Allocation by warm-up.} A mixed-effect Thompson Sampling algorithm has been proposed recently by \citet{aouali2022mixed}, so we can directly use this algorithm as warm up with our given prior parameters. 



\subsection{Robustness to Prior Misspecification}\label{robustness_misspecification}
While the assumption of knowing the prior distribution is common in the Bayesian bandit literature \citep{russo18tutorial,kveton21metathompson}, we acknowledge that it may not always be the case in practice. \cref{lemma:misspecified_prior_Gaussian} quantifies the effect of prior misspecification when using the uniform allocation strategy yielding $\omega = (\frac{1}{K})_{i \in [K]}$, and when the prior variances are homogeneous. Its proof is provided in \cref{subsec:app_proof_prior_misspecification}. 
\begin{lemma}[Upper bound for $\AdaBAI$ with misspecified prior parameters]
\label{lemma:misspecified_prior_Gaussian}
The expected PoE of \emph{$\AdaBAI(\rm{Uni})$} under the MAB problem \eqref{eq:bayes_elim_model_gaussian} with misspecified priors $\cN(\tmu_{0,i},\tsigma^2_{0})$\footnote{Formally, this means that the mean posterior reward for arm $i$ is calculated using the misspecified prior $\cN(\tmu_{0,i},\tsigma^2_{0})$ instead of the true prior $\cN(\mu_{0,i},\sigma^2_{0})$.} is upper bounded as
    \begin{talign*}
        \mathcal{P}_n \leq \doublesum \frac{1}{\sqrt{1 + \frac{n\sigma^2_0}{K\sigma^2}}}\exp\left(-\frac{(\mu_{0,i} - \mu_{0,j})^2}{4\sigma^2_0}\right) d_n(i,j)\,, 
    \end{talign*}
 where
 \vspace{-0.5cm}
\begin{talign*}
d_n(i,j)=\exp\bigg(\frac{\mathrm{KL}(P_{ij}, \Tilde{P}_{ij})}{1+\frac{n\sigma^2_0}{K\sigma^2}}\bigg)\,,
\end{talign*}
with $P_{ij}=\cN(\mu_{0,i}-\mu_{0,j},2\sigma^2_0)$ and $\Tilde{P}_{ij}=\cN(\frac{\tsigma^2_0}{\sigma^2_0}(\tmu_{0,i}-\tmu_{0,j}), 2\sigma^2_0)$ are respectively the true and the misspecified prior gap distributions. Note that 
$\mathrm{KL}(P_{ij}, \Tilde{P}_{ij}) = \frac{\sigma^2_0}{\tsigma^4_0}\big(\frac{\tsigma^2_0}{\sigma^2_0}(\mu_{0,i}-\mu_{0,j})-(\tmu_{0,i}-\tmu_{0,j})\big)^2$.
\end{lemma}
In case of well specified variances, the misspecification factor $d_n(i,j)$  depends only on the true and assumed prior gaps $(\mu_{0,i} - \mu_{0,j})$ and $(\tmu_{0,i} - \tmu_{0,j})$. If the assumed prior preserves the prior gaps, then  $d_n(i,j)=1$, and \cref{th:upper bound} is recovered. However, if the prior gaps are misspecified, then $(d_n(i,j)>1)$, which increases the bound obtained in \cref{th:upper bound}. This multiplicative factor does not affect the overall convergence rate of $\mathcal{O}(1/\sqrt{n})$, and $d_n=e^{\mathcal{O}(1/n)}$, indicating that the impact of prior misspecification becomes less significant as the budget increases. This is also observed empirically (see \cref{subsec:app_prior_misspecification}). That said, we give two solutions to circumvent the assumption of prior knowledge.

\textbf{Offline learning of priors.} Fixed-budget BAI can be significantly enhanced by incorporating readily available offline data as follows. We first estimate arm means using offline data (potentially noisy). 
Then, we model the uncertainty in these estimates with Gaussian covariance, creating a Gaussian prior for multi-armed and linear bandits. For hierarchical bandits, we fit a Gaussian Mixture Model (GMM) to the offline estimates, resulting in the desired hierarchical prior. For instance, in our MovieLens experiments (\cref{sec:experiments}), we derived the offline estimates through low-rank matrix factorization of the user-item ratings. 

\textbf{Online learning of priors using hierarchical models.} In the absence of offline data, one approach to address prior misspecification online is to use hierarchical models. For example, consider the MAB problem \eqref{eq:bayes_elim_model_gaussian}, where the prior means \(\mu_{0}\) are now unknown. Then we adopt the hierarchical model \eqref{eq:model_gaussian}, assuming that \(\mu_{0}\) is sampled from a known hyper-prior as \(\mu_0 \sim \cN(\nu, \Sigma)\), with $L=K$ and canonical mixing weights \(b_i = e_i \in \real^K\). 
This approach to learning the prior online has good empirical performance (\cref{sec:experiments}) and it can be extended to linear and hierarchical bandits.


\section{GENERAL PROOF SCHEME}
\label{sec:analysis}
We outline the key technical insights to derive our Bayesian proofs. The idea is general and can be applied to all our settings (MAB, linear and hierarchical bandits). Specific proofs for these settings are provided in \cref{sec:app_missing_proofs}.

\textbf{From Frequentist to Bayesian proof.} To analyze their algorithm in the MAB setting, \citet{atsidakou2022bayesian} rely on the strong restriction that $\sigma_{0,i}=\sigma_0$ for all arms $i\in [K]$ and tune their allocations as a function of the noise variance $\sigma^2$ such that
in the Gaussian setting, the posterior variances $\hat{\sigma}_{n, i}^2$ in \eqref{eq:bayes_elim_posterior} are equal for all arms $i \in [K]$. This assumption is needed to directly leverage results from \citet{karnin2013almost}, thus allowing them to bound the frequentist PoE as $\mathbb{P}\big(J_n \neq i_*(\theta) \mid \theta \big) \leq B(\theta)$ for a fixed instance $\theta$. 
Then, the expected PoE, $ \mathcal{P}_n = \E{}{\mathbb{P}\big(J_n \neq i_*(\theta) \mid \theta \big)}$, is bounded by directly computing $\E{\theta\sim P_0}{B(\theta)}$. We believe it is not possible to extend such technique for general choices of allocations $n_i$ and prior variances $\sigma^2_{0, i}$. Thus, we pursue an alternative approach, establishing the result in a fully Bayesian fashion. We start with a key observation.

\textbf{Key reformulation of the expected PoE.} We observe that the expected PoE can be reformulated as follows
\begin{align*}
\mathcal{P}_n &=  \E{}{\mathbb{P}\big(J_n \neq i_*(\theta) \mid \theta \big)}
=  \E{}{\mathbb{P}\big(J_n \neq i_*(\theta) \mid H_n \big)}.
\end{align*}
This swap of measures means that to bound $\mathcal{P}_n$, we no longer bound the probability of
$$J_n = \argmax_{i \in [K]}\E{}{r(i; \theta) \,|\, H_n} \neq   \argmax_{i \in [K]} r(i; \theta) = i_*(\theta)$$
\emph{for any fixed $\theta$}. Instead, we only need to bound that probability \emph{when $\theta$ is drawn from the posterior distribution}. Precisely, we bound the probability that the arm $i$ maximizing the posterior mean $\E{}{r(i; \theta) \mid H_n}$ differs from the arm $i$ maximizing the posterior sample $r(i;\theta) \mid H_n$. This is achieved by first noticing that $\mathcal{P}_n$ can be rewritten as (see proof in \cref{subsec:app_proof_MAB})
\begin{align*}
\mathcal{P}_n = \sum_{\substack{i,j \in [K] \\ i \neq j}} \E{}{\mathbb{P}\left(i_*(\theta) = i \mid  J_n = j, H_n\right)\I{J_n = j}}.
\end{align*}
The rest of the proof consists in bounding the above conditional probabilities for distinct $i$ and $j$, and this depends on the setting (MAB, linear or hierarchical). Roughly speaking, this is achieved as follows. $\mathbb{P}\left(i_*(\theta) = i \mid H_n, J_n = j\right)$ is the probability that arm $i$ maximizes the posterior sample $r(\cdot; \theta)$, given that arm $j$ maximizes the posterior mean $\E{}{r(\cdot; \theta) \mid H_n}$. We show that this probability decays exponentially with the squared difference $(\E{}{r(i; \theta)\,|\,H_n} - \E{}{r(j; \theta)\,|\,H_n})^2$. Taking the expectation of this term under the history $H_n$ gives the desired $\mathcal{O}(1/\sqrt{n})$ rate. This proof introduces a novel perspective for Bayesian BAI, distinguished by its tighter prior-dependent bounds on the expected PoE and ease of extension to more complex settings like linear and hierarchical bandits. However, its application to adaptive algorithms could be challenging, particularly due to the complexity of taking expectations under the history $H_n$ in that case. Also, extending this proof to non-Gaussian distributions is an interesting direction for future work, since in general the posterior variances depend on sample observations.

\textbf{Connections with Bayesian Decision Theory. } Note that the `key reformulation' above is inspired from a classical result in Bayesian Decision Theory \citep{robert2007bayesian} that we exploit in a different manner to analyse the Bayes risk of our (non-optimal but computable) decision rule.
\section{EXPERIMENTS}\label{sec:experiments}

\begin{figure*}
\centering
\begin{subfigure}[b]{\textwidth}
   \includegraphics[width=1\linewidth]{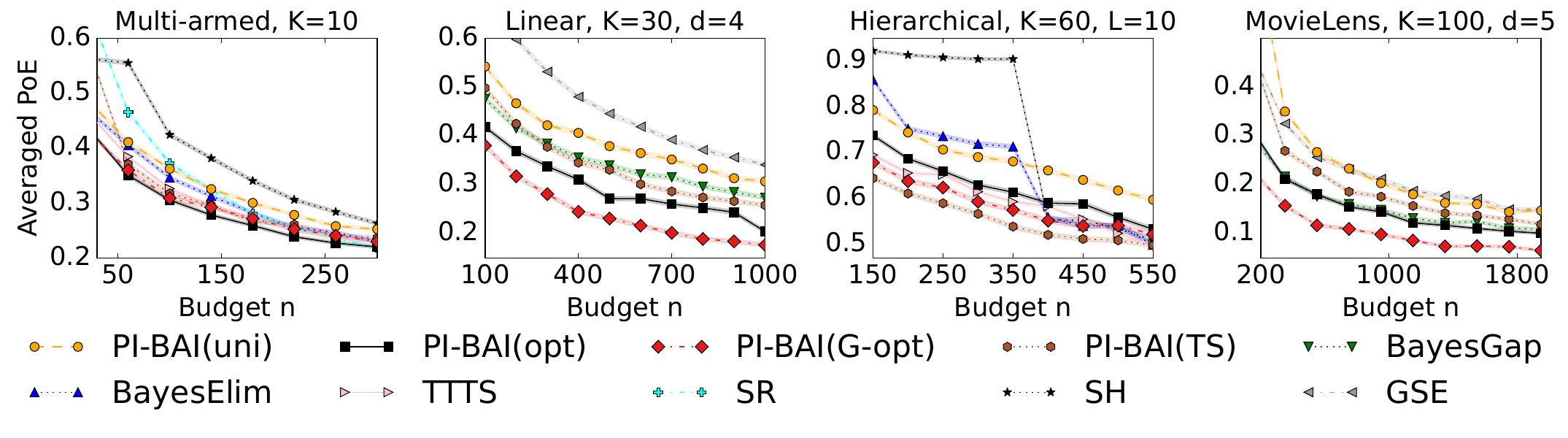}
   \caption{Comparison of \AdaBAI instantiated with different allocation weights to the baselines \BE, \texttt{TTTS}, \texttt{SR}, \texttt{SH} in multi-armed and hierarchical bandits, and to \texttt{BayesGap} and \texttt{GSE} in linear bandits and MovieLens.
   }
   \label{fig:main_standard_experiments} 
\end{subfigure}
\begin{subfigure}[b]{\textwidth}
   \includegraphics[width=1\linewidth]{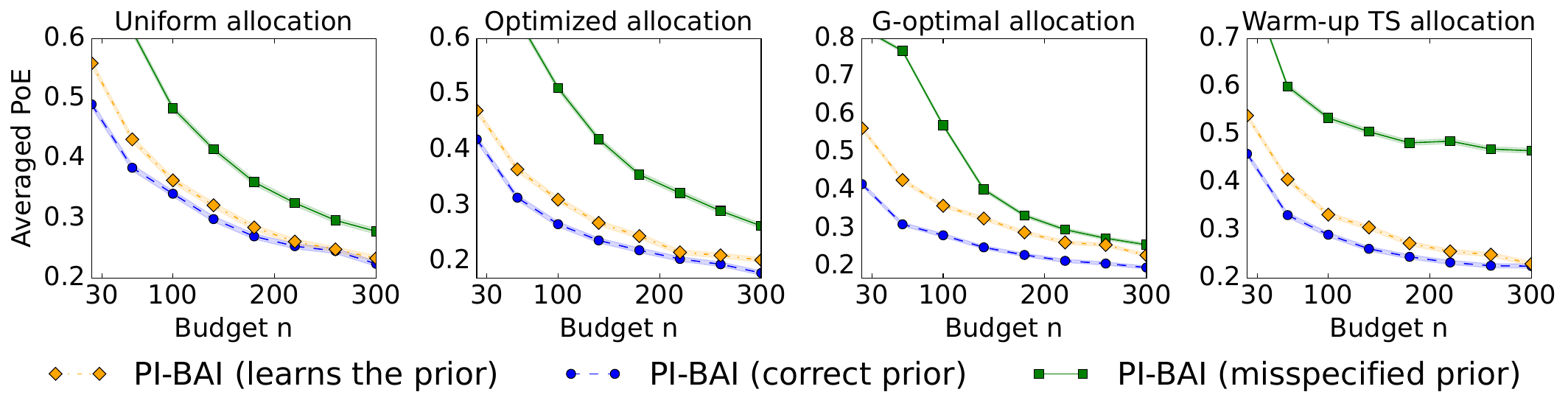}
   \caption{
   For a given allocation strategy, we compare $\AdaBAI$ with learned vs. correct vs. misspecified prior.
   }
   \label{fig:hierarchical_misspecification}
\end{subfigure}
\caption{Average PoE with varying $n$: comparison to baselines and impact of prior misspecification.}
\vspace{-0.3cm}
\end{figure*}
In all experiments, we set the observation noise to $\sigma=1$ and run algorithms $10^4$ times. We display the expected PoE along with its (narrow) standard error. Additional experiments and details are provided in \cref{sec:app_experiments}.
We consider four allocations for $\AdaBAI$ (discussed in \cref{subsec:allocations}): $\AdaBAI(\text{Uni})$ (uniform weights), $\AdaBAI(\text{Opt})$ (optimized weights), $\AdaBAI(\text{G-opt})$ (G-optimal design) and $\AdaBAI(\text{TS})$ (warmed-up with Thompson Sampling) with warm-up length $n_{\rm{w}}=K$. The tuning of $n_{\rm{w}}$ and the choice of the warm-up policy are discussed in \cref{subsec:app_experiments_hyperparameters}. The code is available in this \href{https://github.com/nguyenicolas/Prior_Dependent_Allocations_BAI.git}{Github repository}.


\textbf{MAB.} We set $K=10$, $\mu_{0, i} \sim \mathcal{U}([0, 1])$ and $\sigma_{0,i}$ are evenly spaced between $0.1$ and $0.5$. We compare \AdaBAI to Bayesian methods: top-two Thompson sampling (\texttt{TTTS}\footnote{\texttt{TTTS} does not come with theoretical guarantees in the fixed-budget setting.}) \citep{jourdan2022top,russo2016simple} and \BE \citep{atsidakou2022bayesian}, and to frequentist ones: Successive Rejects (\SR) \citep{audibert2010best} and Sequential Halving (\SH) \citep{karnin2013almost}.

\textbf{Linear bandit.} We set $d=4$, $K=30$, $\mu_{0,i} \sim \mathcal{U}([0, 1])$ and $\Sigma_0 = {\rm{diag}}(\sigma_{0, i}^2)_{i \in [K]}$, with $\sigma_{0,i}$ are evenly spaced between $0.1$ and $0.5$. We compare our methods with two algorithms designed for linear bandits: \BayesGap \citep{hoffman2014correlation} and \GSE \citep{azizi2021fixed}. These represent the current (tractable\footnote{\citet{katz2020empirical} has an algorithm with tighter bounds but it is not tractable.}) state-of-the-art approaches, leveraging G-optimal design for successive elimination. We did not include other methods that use the same elimination approach but have lower performance in these settings \citep{alieva2021robust,yang2022minimax}. 

\textbf{Hierarchical bandit.} We set \(K=60\) and \(L=10\), with mixing weights \(b_i\) sampled uniformly in $[0,1]$ and normalized to sum to 1. We set \(\nu_i \sim \mathcal{U}([-1, 1])\), and \(\Sigma\) and \(\Sigma_0\) are diagonal with entries evenly spaced in \([0.1^2, 0.5^2]\). The MAB frequentist baselines do not require priors, while the Bayesian ones need a Gaussian MAB prior \eqref{eq:bayes_elim_model_gaussian} since they are not suitable for a hierarchical prior. To obtain such a Gaussian MAB prior, we marginalize the effect parameters in \eqref{eq:model_gaussian}. Also, although there is a connection to linear bandits, we do not include those baselines as they do not perform well due to their inefficient representation of the structure, resulting in a high-dimensional ($K+L$) linear model.

\textbf{Results on synthetic data (first three columns of \cref{fig:main_standard_experiments}).} Overall, variants of $\AdaBAI$ outperform the baselines. The best variant depends on the specific setting, but $\AdaBAI(\text{G-opt})$ is the best overall, closely followed by $\AdaBAI(\text{Opt})$. In hierarchical experiments, $\AdaBAI(\text{TS})$ exceeds all baselines, highlighting its effectiveness in capturing the underlying structure. We emphasize that all Bayesian baselines in our experiments (\emph{e.g.}, \BE and \texttt{TTTS}) are implemented using the same informative prior. Overall, these results reaffirm that leveraging prior information is a powerful and practical tool to scale BAI to cases with a large number of arms and limited data. Additional experiments in various settings are in \cref{subsec:app_additional_settings}.

\textbf{MovieLens (fourth column of \cref{fig:main_standard_experiments}).} We also consider the real-world MovieLens dataset \citep{movielens}, which contains 1 million ratings from 6,040 users for 3,883 movies, forming a sparse rating matrix \(M\) of size \(6,040 \times 3,883\). To learn a suitable prior for our algorithm, we complete \(M\) using alternating least squares with rank \(d = 5\) (low-rank factorization), resulting in the decomposition \(M = U^\top V\), where rows \(U_i\) and \(V_j\) represent user \(i\) and movie \(j\), respectively. We then construct a linear Gaussian prior (\cref{subsec:linear_bandit}) by setting the mean \(\mu_0 = \frac{1}{3883} \sum_{j \in [3883]} V_j\) and the covariance \(\Sigma_0= \text{diag}(v)\), where \(v \in \mathbb{R}^d\) is the empirical variance of \(V_j\) for \(j \in [3883]\) along each dimension. 
This prior is fixed across all MovieLens experiments and does not depend on the rewards generated by the bandit instance, resulting in a prior that is not data-dependent.
We run $10^4$ simulations with $K=100$ randomly subsampled movies. In this experiment, the bandit instances are not sampled from a Gaussian prior, yet our algorithm employs a learned Gaussian prior, highlighting the robustness of our method to prior misspecification. Despite this mismatch, \AdaBAI with this prior performs very well, as seen in the fourth column of \cref{fig:main_standard_experiments}.


\textbf{Online learning of priors (\cref{fig:hierarchical_misspecification}).}
We consider the same MAB setting as before with \(K=10\), \(\mu_{0,i}\sim\mathcal{U}([0,1])\), and \(\sigma_{0,i}\) evenly spaced in \([0.1, 0.5]\). We compare three variants of $\AdaBAI$: First, $\AdaBAI$ using the correct prior mean \(\mu_{0,i}\). Second, $\AdaBAI$ with a misspecified prior \(\tilde{\mu}_0 = \mu_0 + 0.5\epsilon\) where \(\epsilon\sim\cN(0, 1)\). Third, $\AdaBAI$ that learns the prior online through the hierarchical model with \(b_i = e_i \in \mathbb{R}^K\), \(\nu_i \sim \mathcal{U}([0, 1])\), and \(\Sigma = I_K\), as described in the last paragraph of \cref{robustness_misspecification}. This comparison is conducted when $\AdaBAI$ variants are instantiated with either Uniform, Optimized, G-optimal or TS allocation weights. As expected, prior misspecification affects $\AdaBAI$'s performance, but this effect diminishes as the budget \(n\) grows, validating our theory. Additionally, $\AdaBAI$ in the hierarchical model effectively learns the correct prior \(\mu_0\), with performance converging to that of $\AdaBAI$ with the correct prior \(\mu_0\). Finally, $\AdaBAI$ with theoretically grounded allocations (Uniform, Optimized, G-optimal) is more robust to prior misspecification than $\AdaBAI$ with TS warm-up, likely because the misspecified prior is used twice: by the warm-up policy TS and by $\AdaBAI$.

\section{CONCLUSION}\label{sec:conclusion}
We revisit the Bayesian fixed-budget BAI for PoE minimization and propose a simple yet effective algorithm for multi-armed, linear, and hierarchical bandits. Our novel proof technique provides a generic, prior-dependent upper bound on the expected PoE that is tighter than existing bounds for MAB. It presents the first bound on the expected PoE for linear and hierarchical bandits. Finally, we address prior misspecification through both theoretical and empirical analysis. Our work illuminates the trade-off between adaptivity and generality in BAI algorithm design and opens several avenues for future research; the competitiveness of this non-adaptive algorithm with elimination schemes is of broader interest, especially as it matches known lower bounds. Theoretically, it would be valuable to explore prior-dependent lower bounds beyond the simple \(K=2\) case.

\subsection*{Acknowledgements}
N. Nguyen and C. Vernade are funded by the Deutsche Forschungsgemeinschaft (DFG) under both the project
468806714 of the Emmy Noether Programme and under Germany’s Excellence Strategy – EXC number 2064/1 – Project number 390727645. Both also thank the international Max Planck Research School for Intelligent Systems
(IMPRS-IS).
\bibliography{biblio}
\bibliographystyle{plainnat}

\newpage
\onecolumn
\section*{Checklist}

 \begin{enumerate}
 \item For all models and algorithms presented, check if you include:
 \begin{enumerate}
   \item A clear description of the mathematical setting, assumptions, algorithm, and/or model. Yes
   \item An analysis of the properties and complexity (time, space, sample size) of any algorithm. Yes
   \item (Optional) Anonymized source code, with specification of all dependencies, including external libraries. Yes
 \end{enumerate}

 \item For any theoretical claim, check if you include:
 \begin{enumerate}
   \item Statements of the full set of assumptions of all theoretical results. Yes
   \item Complete proofs of all theoretical results. Yes
   \item Clear explanations of any assumptions. Yes
 \end{enumerate}

 \item For all figures and tables that present empirical results, check if you include:
 \begin{enumerate}
   \item The code, data, and instructions needed to reproduce the main experimental results (either in the supplemental material or as a URL). Yes
   \item All the training details (e.g., data splits, hyperparameters, how they were chosen). Yes
    \item A clear definition of the specific measure or statistics and error bars (e.g., with respect to the random seed after running experiments multiple times). Yes
    \item A description of the computing infrastructure used. (e.g., type of GPUs, internal cluster, or cloud provider). Yes
 \end{enumerate}

 \item If you are using existing assets (e.g., code, data, models) or curating/releasing new assets, check if you include:
 \begin{enumerate}
   \item Citations of the creator If your work uses existing assets. Yes
   \item The license information of the assets, if applicable. Not Applicable
   \item New assets either in the supplemental material or as a URL, if applicable. Not Applicable
   \item Information about consent from data providers/curators. Not Applicable
   \item Discussion of sensible content if applicable, e.g., personally identifiable information or offensive content. Not Applicable
 \end{enumerate}

 \item If you used crowdsourcing or conducted research with human subjects, check if you include:
 \begin{enumerate}
   \item The full text of instructions given to participants and screenshots. Not Applicable
   \item Descriptions of potential participant risks, with links to Institutional Review Board (IRB) approvals if applicable. Not Applicable
   \item The estimated hourly wage paid to participants and the total amount spent on participant compensation. Not Applicable
 \end{enumerate}
 \end{enumerate}
 
\newpage
\appendix
\section*{ORGANIZATION OF THE APPENDIX}
The supplementary material is organized as follows. In \cref{sec:app_related_work}, we mention additional existing works relevant to our work, and discuss in depth the differences between our work and them. In \cref{sec:app_discussions}, we provide additional general additional remarks. In \cref{sec:app_missing_proofs}, we give complete statements and proofs of our theoretical results. In \cref{sec:app_experiments}, we supply additional numerical experiments.

\section{EXTENDED RELATED WORK}\label{sec:app_related_work}
\textbf{Hierarchical Bayesian bandits.} Bayesian bandits algorithms under hierarchical models have been heavily studied in recent years \citep{aouali2024diffusion,aouali2022mixed,basu21noregrets,hong20latent,hong2022deep,hong22hierarchical,kveton21metathompson,nguyen2023lifelong,peleg2022metalearning}. All the aforementioned papers propose methods to explore efficiently in hierarchical bandit environments to minimize the (Bayesian) regret. Roughly speaking, the idea of all hierarchical Bayesian bandits is to take advantage of the correlations between arms (learned trough the posterior covariance) to have more information of the arms that are not pulled at each round. 
Beyond regret minimization, Bayesian structured models also found success in simple regret minimization \citep{azizi2023meta} and off-policy learning in contextual bandits \citep{aouali2024bayesian,hong2023multi}.

\textbf{Bayesian simple regret minimization and differences with \citet{komiyama2023rate}.} Bayesian simple regret (SR) minimization have been studied in \citet{komiyama2023rate} for Bernoulli bandits. However, we would like to emphasize that our work is very different from theirs. First, SR and PoE may yield to different rate in the Bayesian setting \citep{komiyama2023rate}. While this observation is not true in the frequentist setting \citep{audibert2010best}, \citet{komiyama2023rate} proved an asymptotic $\mathcal{O}(1/n)$ lower bound on the Bayesian SR, while \citet{atsidakou2022bayesian} proved a $\mathcal{O}(1/\sqrt{n})$ lower bound on the PoE. This is because the relationship between SR and PoE is not clear as in the frequentist setting. In fact, we have for a fixed environment $\theta\in\real^K$, 
\begin{align*}
    \mathrm{SR}(\theta) \leq \Delta_{\max}(\theta) \mathbb{P}(J_n \neq i_*(\theta)\mid\theta)\,,
\end{align*}
where $\Delta_{\max}(\theta) = \max_{i\in[K]}\theta_{i} - \min_{i\in[K]}\theta_i$. Integrating with respect to the prior distribution, we bound the expected simple regret as
\begin{align*}
    \E{}{\mathrm{SR}(\theta)}\leq \E{}{\Delta_{\max}(\theta) \mathbb{P}(J_n \neq i_*(\theta)\mid\theta)}\,,
\end{align*}
where the right hand side term involves two non-independent quantities. Note that upper bounding $\Delta_{\max}$ by a pure constant (\emph{e.g.} with high probability for unbounded support reward distributions) would lead to a suboptimal rate of $\mathcal{O}(1/\sqrt{n})$ for the simple regret.

The second difference between these two works is that ours consider the finite-time regime, while theirs derive asymptotic upper and lower bounds. We strongly believe that finite-time guarantees are more interesting in our setting because in practical scenarios, having access to an informative prior is beneficial in the data-poor regime in order to identify the best arm. Moreover, despite deriving Bayesian guarantees, their algorithm is frequentist in the sense that it does not use the prior distribution.

Finally, their work focus on Bernoulli distributions, while our work focus on Gaussian distributions. We believe that it is not straightforward to extend their method and analysis to distributions with unbounded support as Gaussians.

\textbf{Differences on the prior assumption with \citet{atsidakou2022bayesian}.} This latter work considers the K-arms Gaussian setting \eqref{eq:bayes_elim_model_gaussian} and restricts the prior to have equal prior variances $\sigma^2_{0,i}=\sigma^2_{0}$ for all $i\in[K]$. Since their algorithm (\BE) is elimination-based and operates in $R$ rounds, at each elimination round $r\in[R]$, they perform uniform allocation, $n_i = \frac{n}{KR}$. Coupling this with the assumption of prior variance yields to equal posterior covariance, $\sigma^2_{n,i} =\sigma^2_n$ for all $i\in[K]$, \emph{i.e.} arms have the same posterior uncertainty at the end of each round. This observation allows them to use frequentist-based analysis as in \citet{karnin2013almost}. Note that similarly as in frequentist analysis of \citet{karnin2013almost}, their analysis imposes to discard all previously collected observations at the end of each round, impacting the empirical performances of their method. We believe that it is not possible to adapt their analysis to heterogeneous prior variances, where each arm would have a different posterior uncertainty at the end of each round.

We now compare the upper bound of \BE in the setting where their results are valid, that is, when the prior variances are homogeneous, $\sigma^2_{0, i} = \sigma^2_0$. Their bound and ours read
\begin{talign*}
\mathcal{P}^{\rm{\BE}}_n \leq \sum_{\substack{i,j \in [K] \\ i \neq j}} \textcolor{blue}{\log_2(K)} \frac{e^{-\frac{(\mu_{0, i}-\mu_{0, j})^2}{4\sigma^2_0}}}{\sqrt{1 + \frac{n\sigma^2_0}{K\textcolor{blue}{\log_2(K)}\sigma^2}}}\,, \, \text{ and } \, \, \mathcal{P}^{\AdaBAI}_n \leq  \sum_{\substack{i,j \in [K] \\ i \neq j}}\frac{ \exp\big(-\frac{(\mu_{0, i}-\mu_{0, j})^2}{4\sigma^2_0}\big)}{\sqrt{1 + \frac{n\sigma^2_0}{K\sigma^2}}}\,.
\end{talign*}
By omitting elimination phases, we gain roughly a factor $\log^{3/2}_2(K)$ over the bound of \citet{atsidakou2022bayesian} (highlighted in blue). This additional factor can be significant when bounding a probability ($\log^{3/2}_2(K)\approx 6$ for $K=10$). This difference makes our bound smaller even in their setting with homogeneous prior variances and choosing uniform allocation weights for \AdaBAI.

\section{ADDITIONAL DISCUSSIONS}\label{sec:app_discussions}

\subsection{Motivating Practical Examples for Hierarchical Bandits}
\label{subsec:app_motivating_examples}
We discuss motivating examples for using hierarchical models in pure exploration settings. 

\textbf{Hyper-parameter tuning.} Here, the goal is to find the best configuration for a neural network using $n$ epochs \citep{li2017hyperband}. A configuration $i$ is represented by a scalar $\theta_i \in \real$ which quantifies the expected performance of a neural network with such configuration. Again, it is intuitive to learn all $\theta_i$ individually. Roughly speaking, this means running each configuration for $\lfloor \frac{n}{K} \rfloor$ epochs and selecting the one with the highest performance. This is statistically inefficient since the number of configurations can be combinatorially large. Fortunately, we can leverage the fact that configurations often share the values of many hyper-parameters. Precisely, a configuration is a combination of multiple hyper-parameters, each set to a specific value. Then we represent each hyper-parameter $\ell \in [L]$ by a scalar parameter $\mu_\ell \in \real$, and the configuration parameter $\theta_i$ is a mixture of its hyper-parameters $\mu_\ell$ weighted by their values. That is, $\theta_i = \sum_{\ell \in [L]} b_{i,\ell} \mu_\ell + \epsilon_i$, where $b_{i, \ell} $ is the value of hyper-parameter $\ell$ in configuration $i$ and $\epsilon_i$ is a random noise to incorporate uncertainty due to model misspecification.  

\textbf{Drug design.} In clinical trials, $K$ drugs are administrated to $n$ subjects, with the goal of finding the optimal drug design. Each drug is parameterized by its expected efficiency $\theta_i \in \mathbb{R}$. As in the previous example, it is intuitive to learn each $\theta_i$ individually by assigning a drug to $\lceil \frac{n}{K}\rceil$ subjects. However, this is inefficient when $K$ is large. We leverage the idea that drugs often share the same components: each drug parameter $\theta_i$ is a combination of component parameters $\mu_l$, each accounting for a specific dosage. More precisely, the parameter of drug $i$ can be modeled as $\theta_i = \sum_{l \in [L]} b_{i, l}\mu_l + \epsilon_i$ where $\epsilon_i$ accounts for uncertainty due to model misspecification. Similarly to the hyper-parameter tuning example, this models correlations between drugs and it can be leveraged for more efficient use of the whole budget of $n$ epochs.

\subsection{Additional Remarks on Hierarchical Models}
\label{subsec:app_remarks_hierarchical_models}
The two-level prior we consider has a shared latent parameter $\mu = (\mu_\ell)_{\ell \in [L]} \in \real^L$ representing $L$ \emph{effects} impacting each of the $K$ arm means:  
\begin{talign*}
    \mu &\sim Q \nonumber\\
    \theta_i &\sim P_{0, i}(\cdot ; \mu)  & \forall i \in [K]\nonumber\\
    Y_t \mid \mu, \theta, A_t &\sim P(\cdot; \theta_{A_t})  & \forall t \in [n]\,,
\end{talign*}
where $Q$ is the latent prior on $\mu \in \real^L$.

In the Gaussian setting \eqref{eq:model_gaussian}, the maximum likelihood estimate of the reward mean of action $i$, $B_{n, i}/n_i$, contributes to \eqref{eq:cluster_posterior} proportionally to its precision $n_i / (n_i \sigma^2_{0, i} + \sigma^2)$ and is weighted by its mixing weight $b_i$. \eqref{eq:conditional_posterior} is a standard Gaussian posterior, and \eqref{eq:posterior} takes into account the information of the conditional posterior. Finally, \eqref{eq:posterior} takes into account the arm correlation through its dependence on $\hat\sigma_{n,i}$ and $\hat\mu_{n,i}$. While the properties  of conjugate priors are useful for inference, other models could be considered with approximate inference techniques \citep{clavier2023vits,phan2019thompson}.

\textbf{Link with linear bandit (cont.).}
The slightly unusual characteristic of \eqref{eq:first_lb} is that the prior distribution has correlated components. This can be addressed by the whitening trick \citep{Bishop2006}, defining
$\tmu = \bSigma^{-1/2} \bmu$ and $\tb_i = \bSigma^{1/2} \bb_i$, giving
\begin{talign}\label{eq:second_lb}
    \tmu &\sim \cN(\bSigma^{-1/2} \bnu, I_{L+K})\nonumber\\
    Y_t \mid \tmu, A_t &\sim \cN(\tb^\top_{A_t} \tmu, \sigma^2)  & \forall t \in [n]\,,
\end{talign}
where $I_{L+K}$ is the $(L+K)$-dimensional identity matrix. Then, \eqref{eq:second_lb} corresponds to a linear bandit model with $K$ arms and $d=K+L$ features. However, this model comes with some limitations. First, when computing posteriors under \eqref{eq:second_lb}, the time and space complexities are $\mathcal{O}((K+L)^3)$ and $\mathcal{O}((K+L)^2)$ respectively, compared to the $\mathcal{O}(K+L^3)$ and $\mathcal{O}(K+L^2)$ for our model \eqref{eq:posterior}.  
The feature dimension $d=K+L$ can be reduced to $d=K$ through the following QR decomposition: $\Tilde{B} = \big(\Tilde{b}_1,\dots,\Tilde{b}_K\big)\in \mathbb{R}^{(K+L) \times K}$ can be expressed as $\Tilde{B} = V R$, where $V\in\mathbb{R}^{(K+L) \times K}$ is an orthogonal matrix and $R\in \mathbb{R}^{K \times K}$. This leads to the following model $\check{\mu} \sim \cN(V^\top\bSigma^{-1/2} \bnu, V^\top V)$ and $Y_t \mid \check{\mu}, A_t \sim \cN(R^\top_{A_t}\check{\mu}, \sigma^2)$, yet the feature dimension $d$ remains at the order of $K$, and computational efficiency is not improved with respect to $K$. 

\textbf{From hierarchical bandit to MAB.}
Marginilizing the hyper-prior in \eqref{eq:model_gaussian} leads to a MAB model,
\begin{talign}
\label{eq:from_hier_to_mab}
    \theta_i &\sim \cN(b_i^\top \nu, \sigma^2_{0, i}+ b_i^\top \Sigma b_i) &\forall i \in [K] \nonumber\\
    Y_t \mid \mu, \theta, A_t &\sim \cN(\theta_{A_t}, \sigma^2) &\forall t \in [n]\,.
\end{talign}
In this marginalized model, the agent does not know $\mu$ and he doesn't want to model it. Therefore, only $\theta$ is learned. The marginalized prior variance $\sigma^2_{0, i} + b_i^\top \Sigma b_i$ accounts for the uncertainty of the not-modeled effects. From \eqref{eq:bayes_elim_posterior}, the corresponding posterior covariance of an arm $i \in [K]$ is
\begin{align*}
    \hat\sigma^{-2}_{n, i} = \frac{1}{\sigma^2_{0, i} + b_i^T \Sigma b_i} + \frac{\omega_i n}{\sigma^2}\,.
\end{align*}

We illustrate the benefit of using hierarchical models over MAB models with a toy experiment. In the first setting, we uniformly draw a vector $u \in [0, 1]$ and set $\sigma_0 = 0.1u$ and $\Sigma = 2I_L$.
In the second setting, we set $\sigma_0 = u$ and $\Sigma = 10^{-3}I_L$. In both settings, we consider $K=50$ arms, and $L=10$ effects. Each $\nu_i$ and $b_i$ are  sampled from $[-1, 1]$, and the allocation vector is set to uniform allocation, $\omega^{\rm{uni}}_i = \frac{1}{K}$ for any $i \in [K]$. \cref{fig:app_benefits_hier} shows the average posterior covariance $\frac{1}{K}\sum_{i \in [K]}\sigma^2_{n, i}$ across all arms for both the (marginalized) standard MAB model \eqref{eq:from_hier_to_mab} in blue and the hierarchical model \eqref{eq:model_gaussian} in red. The results show that the benefits of using hierarchical models are more pronounced when the initial uncertainty of the effects $\Sigma$ is greater than the initial uncertainty of the mean rewards $(\sigma^2_{0, i})_{i \in [K]}$.

\begin{figure}[H]
\centering
\includegraphics[width=10cm]{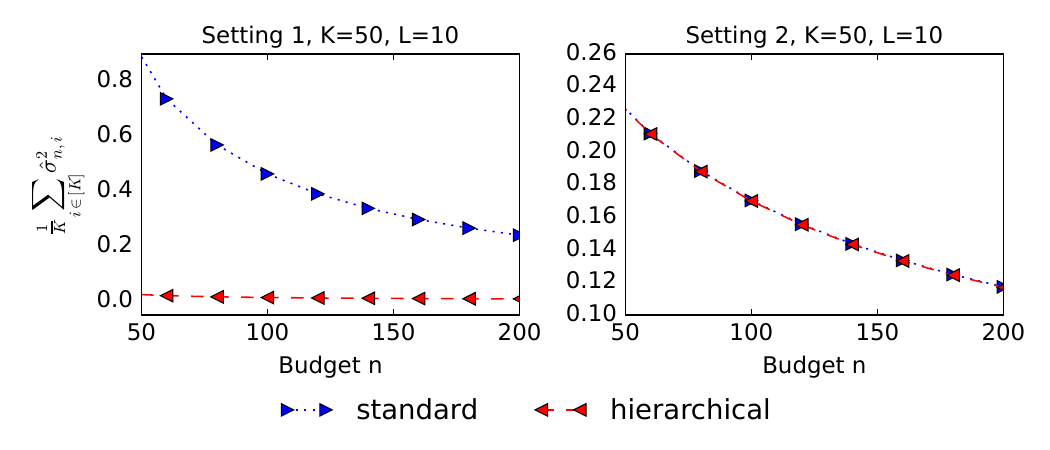}
\caption{Average posterior covariance across all arms for standard MAB and hierarchical model for two settings.} 
\label{fig:app_benefits_hier}
\end{figure}

\subsection{Beyond Gaussian Distributions}
\label{subsec:app_remarks_linear_models}

\textbf{Beyond linear models.} 
The standard linear model \eqref{eq:linear_gaussian} can be generalized beyond linear mean rewards. The \emph{Generalized Linear Bandit} (GLB) model with prior $P_0$ writes \citep{filippi10parametric,kveton2020randomized}
\begin{align}\label{eq:GLB_model}
   &\theta \sim P_0  \\
    &Y_t\mid\theta, A_t \sim P(.;\theta, A_t) \quad \forall t \in [n]\,,\nonumber
\end{align}
where the reward distribution $P(.;\theta, A_t)$ belongs to some exponential family with mean reward $r(A_t;\theta) = \phi(\theta^\top x_{A_t})$. $\phi$ is called the \emph{link function}. The \emph{log-likelihood} of such reward distribution can be written as
\begin{align*}
    \mathcal{L}_n(\theta) = \sum_{t=1}^n \log \mathrm{P}(Y_t;\theta, A_i) = \sum_{t=1}^n Y_t \theta^\top x_{A_t} - A(\theta^\top x_{A_t}) + h(Y_{t})\,,
\end{align*}
where $A$ is a \emph{log-partition} function and $h$ another function. Importantly, \eqref{eq:GLB_model} encompasses the \emph{logistic bandit} model with the particular link function $\phi(z) = \frac{1}{1 + e^{-z}}\,$. 

The main challenge of \eqref{eq:GLB_model} is that closed-form posterior generally does not exist. One method is to approximate the posterior distribution of $\theta$ given $H_n$ with Laplace approximation, that is, $\mathbb{P}(\theta \mid H_n)$ is approximated with a multivariate Gaussian distribution with mean $\thetaMAP$ and covariance  $\hat\Sigma_{\mathrm{Lap}}$,
\begin{align*}
    &\thetaMAP = \argmax_{\theta}\mathcal{L}_n(\theta)P_0(\theta)\,, &\hat\Sigma^{-1}_{\mathrm{Lap}} = \sum_{t=1}^n \dot{\phi}(\thetaMAP^\top x_{A_t}) x_{A_t}x_{A_t}^\top\,,
\end{align*}
where $\phi$ is assumed continuously differentiable and increasing. Note that $\thetaMAP$ can be computed efficiently by iteratively reweighted least squares \citep{wolke88iteratively}. 

\textbf{Logistic Bandit.}
In the particular case where the reward distribution is Bernoulli, the model writes
\begin{talign}\label{eq:glb_gaussian}
   \theta &\sim \cN(\mu_{0}, \Sigma_0)  &\nonumber\\
    Y_t \mid \theta, A_t &\sim \mathcal{B}(\phi(\theta^\top x_{A_t}))  & \forall t \in [n]\,,
\end{talign}
where $\phi$ is the logistic function. Then, using a result from \citet{spiegelhalter1990sequential}, the mean posterior reward can be approximated as 
\begin{align*}
    \E{\theta \sim \cN(\thetaMAP, \hat\Sigma_{\mathrm{Lap}})}{\phi(\theta^\top x_i)} \approx \frac{\phi(\thetaMAP^\top x_i)}{\sqrt{1+ \frac{\pi}{8} \normw{x_i}{\hat\Sigma_{\mathrm{Lap} }} }}\,,
\end{align*} 
and we set the decision after $n$ rounds as $J_n = \argmax_{i \in [K]} \frac{\phi(\thetaMAP^\top x_i)}{\sqrt{1+ \frac{\pi}{8} \normw{x_i}{\hat\Sigma_{\mathrm{Lap} }} }}\,$.

Proving an upper bound on the expected PoE of this algorithm is challenging. Particularly, upper bounding the expectation with respect to $H_n$ is hard because one needs to show that $\thetaMAP$ concentrates in norm towards its expectation $\E{H_n}{\thetaMAP}$. We leave this study for future work. However, we provide numerical experiments for this setting in \cref{subsec:experiments_GLB}.

\subsection{Differences Between Optimized and Warmed-up Weights.}
\label{subsec:app_differences_opt_and_TS}
To illustrate the differences between optimized weights ($\AdaBAI(\wOpt)$) and learned weights with Thompson sampling as a warm-up policy ($\AdaBAI(\wTS)$), we return to our motivating example in \cref{subsec:bound_MAB}, where $K=3$, $\mu_{0}=(1, 1.9, 2)$ and $\sigma_{0, i}=0.3$ for all $i \in [3]$. We set the budget as $n=100$. We repeat $10^4$ times the following experiment: we sample a bandit instance from the prior and run Thompson sampling for $n_{\rm{w}}=20$ rounds, then construct the allocation weights $\wTS$. Computing the weights $\wOpt$ by numerical optimization of \eqref{eq:optimized_weights} is done once at the beginning of these experiments.

\cref{fig:toy_example_allocations} shows an empirical comparison of the weights on 2 problem instances and on average over $10^4$ runs. We see that, in this example, both allocation strategies assign high weights to arms 2 and 3, while allocating a small weight to arm 1. This is because, based on the prior information, arm 1 is highly unlikely to be the optimal arm. Then, the primary objective revolves around selecting the optimal arm among arms 2 and 3. Also, while the allocation weights $\omega^{\textsc{ts}}_i$ vary with each bandit instance, their average values in all instances are similar to those of $\wOpt_i$. Thus $\AdaBAI(\wTS)$ is more adaptive than $\AdaBAI(\wOpt)$, while both have similar average behavior. 
\begin{figure}
  \centering
  \includegraphics[width=10cm]{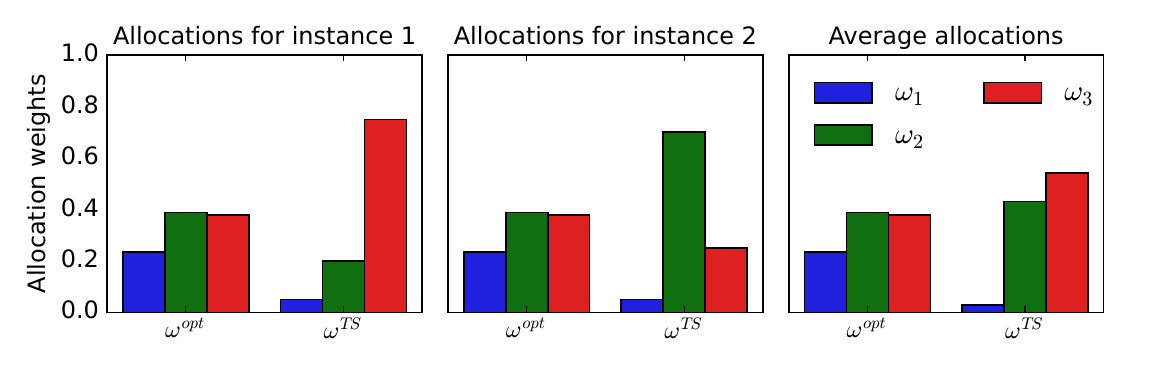}
  \caption{Allocations weights $\wOpt_i$ and $\wTS_i$.} 
  \label{fig:toy_example_allocations}
\end{figure}

\section{MISSING PROOFS}\label{sec:app_missing_proofs}
In this section, we give complete proof of our theoretical results. In \cref{subsec:app_proof_technical}, we give proofs for the Bayesian posterior derivations and we provide technical results. Then we provide the complete proofs of our results for MAB (\cref{subsec:app_proof_MAB}), linear bandits (\cref{subsec:app_proof_linear_bandit}) and hierarchical bandits (\cref{subsec:app_proof_hierarchical}).

\subsection{Technical Proofs and Posteriors Derivations}
\label{subsec:app_proof_technical}

\textbf{Bayesian computations for hierarchical model.}
We detail the computations of posterior distribution computations of the hierarchical Gaussian model,
\begin{align*}
    \mu &\sim \cN(\nu, \Sigma)\nonumber\\
    \theta_i &\sim \cN(b_i^\top \mu, \sigma_{0, i}^2)  & \forall i \in [K]\nonumber\\
    Y_t \mid \mu, \theta, A_t &\sim \cN(\theta_{A_t}, \sigma^2)  & \forall t \in [n]\,,
\end{align*} 
where we recall that $B_{n, i} = \sum_{t\in \mathcal{T}_i}Y_t$ and $\mathcal{T}_i = \{t\in [n],\, A_t=i\}$.

\begin{lemma}[Gaussian posterior update]
\label{lemma:gaussian_posterior_update}
For any $\rho \in \mathbb{R}, \mu \in \mathbb{R}^L, b \in \mathbb{R}^L$ and $\sigma, \sigma_{0} > 0, m \in \mathbb{N}$, we have
\begin{align*}
    \int_{\rho} \prod_{t \in [m]} \cN\left( Y_t ; \rho, \sigma^2 \right)\cN\left(\rho ; b^\top \mu, \sigma^2_{0}\right)\dint\rho \propto \cN\left(\mu ; \mu_m, \Sigma_m\right)\,,
\end{align*}
with
\begin{align*}
    \Sigma_m = \frac{m}{m\sigma^2_{0}+\sigma^2}b b^\top ; \quad \mu_m = \Sigma_m^{-1} \frac{\sum_{t \in [m]}Y_t}{m\sigma^2_0 + \sigma^2} b\,.
\end{align*}
\end{lemma}
\begin{proof}[Proof of \cref{lemma:gaussian_posterior_update}]
By keeping only terms that depend on $\mu$,
\begin{align*}
    f(\mu) &= \int_{\rho} \prod_{t \in [m]} \cN\left( Y_t ; \rho, \sigma^2 \right)\cN\left(\rho ; b^\top \mu, \sigma^2_{0}\right)\dint\rho \\
    &\propto  \int_\rho\exp\left\{ -\frac{1}{2\sigma^2}\sum_{t \in [m]}(Y_t - \rho)^2 - \frac{1}{2\sigma^2_{0}}(\rho - b^\top\mu)^2 \right\} \dint\rho\\
    &\propto \int_\rho \exp\left\{ -\frac{1}{2}\rho^2 \left(\frac{1}{\sigma^2_0} + \frac{m}{\sigma^2}\right) - 2\rho\left(\frac{\sum_{t \in [m]}}{\sigma^2} + \frac{b^\top\mu}{\sigma^2_0}\right) \right\}\dint\rho \exp\left\{-\frac{1}{2\sigma^2_0}\mu^T b b^\top \mu\right\}\\
    &\propto \exp\left\{ \frac{1}{2}\left(\frac{1}{\sigma^2}\sum_{t\in [m]}Y_t + \frac{b^\top \mu}{\sigma^2_0}\right)^2 \frac{\sigma^2_0 \sigma^2}{\sigma^2 + m\sigma^2_0} -\frac{1}{2\sigma^2_0}\mu^T b b^\top \mu\right\}\\
    &\propto \exp\left\{ \frac{\sum_{t\in [m]}Y_t b^\top \mu}{\sigma^2 + m\sigma^2_0}+\frac{\sigma^2}{2(\sigma^2 + m\sigma^2_0)}\mu^\top b b^\top \mu-\frac{1}{\sigma^2_0}b b^\top \right\}\\
    &\propto \exp\left\{ -\frac{1}{2}\left(\mu^\top \frac{m}{\sigma^2 + m\sigma^2_0}b b^\top\mu - 2\mu^\top  \frac{\sum_{t \in [m]}Y_t}{\sigma^2 + m\sigma^2_0} b \right)\right\}\\
    &\propto \cN\left( \mu ; \mu_m, \Sigma_m\right)\,.
\end{align*}
\end{proof}

\begin{lemma}[Joint effect posterior] \label{lemma:joint_effect_posterior}
For any $n \in [N]$, the joint effect posterior is a multivariate Gaussian $Q_n(\mu) = \cN(\breve{\mu}, \breve{\Sigma}_n)$, where
\begin{align}\label{eq:cluster_posterior}
    \breve{\Sigma}_n^{-1} = \Sigma^{-1} + \sum_{i\in[K]} \frac{n_i}{n_i \sigma_{0, i}^2 + \sigma^2} b_i b_i^\top, \quad 
    \breve{\mu}_n = \breve{\Sigma}_n\Big( \Sigma^{-1} \nu + \sum_{i\in[K]} \frac{B_{n, i}}{n_i\sigma_{0, i}^2+\sigma^2}  b_i\Big)\,.
\end{align}
\end{lemma}

\begin{proof}[Proof of \cref{lemma:joint_effect_posterior}]
The joint effect posterior can be written as
\begin{align*}
    Q_n(\mu) &\propto \int_\theta \mathcal{L}_\theta\left(Y_{A_1}, ..., Y_{A_n}\right)P_0(\theta \mid \mu) \dint\theta Q(\mu)\\
    &= \prod_{i \in [K]}\int_{\theta_i} \prod_{t \in \mathcal{T}_i} \cN\left(Y_t ; \theta_i, \sigma^2\right)\cN\left( \theta_i; b_i^\top\mu, \sigma^2_{0, i} \right)\dint\theta_i \cN\left( \mu ; \nu, \Sigma \right)\,.
\end{align*}
Applying \cref{lemma:gaussian_posterior_update} gives
\begin{align*}
    \int_{\theta_i} \prod_{t \in \mathcal{T}_i} \cN\left(Y_t ; \theta_i, \sigma^2\right)\cN\left( \theta_i; b_i^\top\mu, \sigma^2_{0, i} \right)\dint\theta_i &\propto \cN\left(\mu;\bmu_{n, i}, \bSigma_{n, i}\right)\,,
\end{align*}
with
\begin{align*}
\breve\Sigma_{n, i}^{-1} = \frac{n_i}{\sigma^2_{0, i}n_i + \sigma^2}b_i b_i^\top, \qquad \breve\mu_{n, i} =  \breve\Sigma_{n, i} b_i\frac{n_i}{\sigma^2_{0, i}n_i + \sigma^2}\frac{B_{n, i}}{n_i}\,.
\end{align*}
Therefore, the joint effect posterior is a product of Gaussian distributions,
\begin{align*}
    Q_n(\mu) \propto \prod_{i \in [K]} \cN\left(\mu;\breve\mu_{n, i}, \breve\Sigma_{n, i}\right)\cN\left(\mu;\nu, \Sigma\right)\propto \cN\left( \mu ; \breve\mu_n, \breve\Sigma_n \right)\,,
\end{align*}
where
\begin{align*}
     \breve\Sigma_n &= \Sigma^{-1} + \sum_{i \in [K]}\breve\Sigma_{n, i} = \Sigma^{-1} + \sum_{i \in [K]}\frac{n_i}{\sigma^2_{0, i}n_i + \sigma^2}b_i b_i^\top\\
     \breve\mu_{n} &= \breve\Sigma_n^{-1}\left(\Sigma^{-1}\nu + \sum_{i \in [K]}\breve\Sigma_{n,i}^{-1}\breve\mu_{n, i} \right)= \breve\Sigma_n^{-1} \left(\Sigma^{-1}\nu + \sum_{i \in [K]}\frac{B_{n, i}}{n_i \sigma^2_{0, i}+\sigma^2}b_i\right)\,.
\end{align*}
\end{proof}

\begin{lemma}[Conditional arm posteriors]
\label{lemma:Conditional_arm_posteriors}
For any $n \in [n]$ and any arm $i \in [K]$, the conditional posterior distribution of arm $i$ is a Gaussian distribution $P_{n, i}(\theta_i\mid\mu) = \cN\left( \tmu_{n, i}, \tsigma^2_{n, i} \right)$, where
\begin{align}\label{eq:conditional_posterior}
    \tilde{\sigma}_{n, i}^{-2} = \frac{1}{\sigma_{0, i}^2}  + \frac{n_i}{\sigma^2}\,, \quad \tilde{\mu}_{n, i} = \tilde{\sigma}_{n, i}^{2} \left( \frac{\mu^\top b_i}{\sigma_{0, i}^{2}} +\frac{B_{n, i}}{\sigma^2} \right)\,.
\end{align}
\begin{proof}[Proof of \cref{lemma:Conditional_arm_posteriors}]
    The conditional posterior of arm $i$ can be written as
    \begin{align*}
        P_{n, i}(\theta_i\mid\mu) &\propto \mathcal{L}_{\theta_i}(Y_{A_1}, ...,Y_{A_n} )P_{0, i}(\theta_i \mid \mu)\\
        &\propto \prod_{t \in \mathcal{T}_i} \cN\left(Y_t ; \theta_i, \sigma^2\right)\cN\left(\theta_i ; b_i^T \mu, \sigma^2_{0, i} \right)\\
        &\propto \exp\left\{ -\frac{1}{2\sigma^2}\sum_{t \in \mathcal{T}_i}(Y_t - \theta_i)^2 - \frac{1}{2\sigma^2_{0, i}}(\theta_i - b_i^\top\mu)^2 \right\}\\
        &\propto \exp\left\{ -\frac{1}{2\sigma^2}\sum_{t \in \mathcal{T}_i}\left(-2 Y_t \theta_i + \theta_i^2\right) - \frac{1}{2\sigma^2_{0, i}}\left( \theta_i^2 - 2\theta_i b_i^\top\mu \right) \right\}\\
        &\propto \exp\left\{ -\frac{1}{2} \left(\theta_i^2 \left(\frac{n_i}{\sigma^2} + \frac{1}{\sigma^2_{0, i}}\right) - 2\theta_i \left(\frac{1}{\sigma^2}\sum_{t\in \mathcal{T}_i}Y_t + \frac{1}{\sigma^2_{0, i}}b_i^\top\mu\right)\right) \right\}\\
        &\propto \cN\left(\theta_i;\tmu_{n, i}, \tsigma^2_{n, i}\right)
    \end{align*}
\end{proof}
\end{lemma}

\begin{lemma}[Marginal arm posterior]
\label{lemma:Marginal_arm_posterior}
For any $n \in [n]$ and any arm $i \in [K]$, the marginal posterior distribution of arm $i$ is a Gaussian distribution  $\mathbb{P}(\theta_i \mid H_n) = \cN\left(\hat\mu_{n, i}, \hat\sigma^2_{n, i}\right)$, where
\begin{align}\label{eq:posterior}
    \hat{\sigma}_{n, i}^{2} = \tilde{\sigma}_{n, i}^{2} + \frac{\tilde{\sigma}_{n, i}^{4}}{\sigma_{0, i}^4} b_i^\top \breve{\Sigma}_n b_i, \quad \hat{\mu}_{n, i} = \tilde{\sigma}_{n, i}^{2} \Big( \frac{\breve{\mu}_n^\top b_i}{\sigma_{0, i}^{2}} +\frac{B_{n, i}}{\sigma_i^2} \Big)\,.
\end{align}
\end{lemma}
\begin{proof}[Proof of \cref{lemma:Marginal_arm_posterior}]
The marginal distribution of arm $i$ can be written as
\begin{align*}
\int_\mu P_{n, i}(\theta_i \mid \mu) Q_n(\mu)d\mu &= \int_\mu \cN\left(\theta_i;\tmu_{n, i}, \tsigma^2_{n, i}\right) \cN\left( \mu;\breve\mu_n, \breve\Sigma_n \right)\dint\mu\\
&\propto \int_\mu \cN\left(\theta_i;\tsigma^2_{n, i}\left(\frac{\mu^\top b_i}{\sigma^2_{0, i}}+\frac{B_{n, i}}{\sigma^2}\right), \tsigma^2_{n, i}\right) \cN\left( \mu;\breve\mu_n, \breve\Sigma_n \right)\dint\mu\,.
\end{align*}
The line above is a convolution of Gaussian measures, and can be written as \citep{Bishop2006},
\begin{align*}
\int_\mu P_{n, i}(\theta_i \mid \mu) Q_n(\mu)d\mu &\propto \cN\left(\theta_i ; \tilde{\sigma}_{n, i}^{2} \left( \frac{\breve{\mu}_n^\top b_i}{\sigma_{0, i}^{2}} +\frac{B_{n, i}}{\sigma_i^2} \right), \tsigma_{n, i}^2 + \frac{\tsigma_{n, i}^2}{\sigma^2_{0, i}}b_i^\top \breve\Sigma_n\frac{\tsigma_{n, i}^2}{\sigma^2_{0, i}}b_i \right)\dint\mu\\
&= \cN\left(\theta_i ; \hat\mu_{n, i}, \hat\sigma_{n, i}^2\right)\,.
\end{align*}
\end{proof}
\begin{lemma}[Technical lemma]
\label{lemma:integration_gaussian}
Let $a >0$ and $X \sim \cN(\mu, \sigma^2)$. Then $\E{X}{e^{-\frac{X^2}{2a^2}}}= \frac{1}{\sqrt{1 + \frac{\sigma^2}{a^2}}} e^{-\frac{\mu^2}{2(a^2 + \sigma^2)}}$.
\end{lemma}

\subsection{Proofs for Multi-armed Bandits}
\label{subsec:app_proof_MAB}

From now, we consider that $\lfloor \omega_k n\rfloor = \omega_k n \in \mathbb{N}$ for sake of simplicity. 

\begin{theorem}[Complete statement of \cref{th:upper bound}] For all $\omega \in \Delta^+_K$, the expected PoE of \emph{$\AdaBAI(\omega)$} under the MAB problem \eqref{eq:bayes_elim_model_gaussian} is upper bounded as
\begin{align}\label{eq:MAB_phi}
  \mathcal{P}_n &\leq  \doublesum\frac{e^{ -\frac{(\mu_{0, i} - \mu_{0, j})^2}{2(\sigma^2_{0, i} + \sigma^2_{0, j})}}}{\sqrt{1 + n\phi_{i, j}(\omega)}}\,,\quad \text{where}\quad
    \phi_{i, j}(\omega) = 
    \frac{\sigma^4_{0, i}\omega_i\left(\frac{\sigma^2}{n} + \omega_{j}\sigma^2_{0, j}\right) + \sigma^4_{0, j}\omega_{j}\left(\frac{\sigma^2}{n} + \omega_{i}\sigma^2_{0, i}\right)}{\sigma^2 \sigma^2_{0, i}\left(\frac{\sigma^2}{n} + \omega_{j}\sigma^2_{0, j}\right) + \sigma^2 \sigma^2_{0, j}\left(\frac{\sigma^2}{n} + \omega_{i}\sigma^2_{0, i}\right)}\,.
\end{align}
\end{theorem}

\begin{remark}
    When $\sigma^2_{0, i} = \sigma^2_{0, j}\,,\lim_{n \to +\infty} \phi_{i, j} = \lim_{n \to \infty} \frac{\sigma^4_{0, i}\omega_i\left(\frac{\sigma^2}{n} + \omega_{j}\sigma^2_{0, j}\right) + \sigma^4_{0, j}\omega_{j}\left(\frac{\sigma^2}{n} + \omega_{i}\sigma^2_{0, i}\right)}{\sigma^2 \sigma^2_{0, i}\left(\frac{\sigma^2}{n} + \omega_{j}\sigma^2_{0, j}\right) + \sigma^2 \sigma^2_{0, j}\left(\frac{\sigma^2}{n} + \omega_{i}\sigma^2_{0, i}\right)} = \frac{2\sigma^2_0}{\sigma^2}\frac{\omega_i \omega_j}{\omega_i+\omega_j}=\mathcal{O}(1)$.
\end{remark}

\textbf{Proof of \cref{th:upper bound}}
We first write $\mathcal{P}_n$ as a double sum over all possible distinct arms,
\begin{align*}
    \E{}{\mathbb{P}_{}\big(J_n \neq i_*(\theta) \mid H_n \big)} &= \E{}{\I{J_n \neq i_*(\theta)}}\\
    &= \E{}{\E{}{\sum_{i=1}^K \sum_{j=1}^K \I{i \neq j} \I{i_*(\theta) = i} \I{J_n = j} \mid H_n}}\\
    &=\doublesum\E{}{\E{}{\I{i_*(\theta) = i} \I{J_n = j} \mid H_n}}\\
    &=\doublesum\E{}{\mathbb{P}_{}\left(i_*(\theta) = i \cap  J_n = j \mid H_n\right)}
\end{align*}
Since $J_n : H_n \to [K],\,\mathbb{P}(J_n = j \mid H_n) = \I{J_n = j}$. For the conditional probability of $i_*(\theta) = i$ given the event $J_n = j$ to be well defined, we want to make sure to condition on an event with non-zero probability. Considering both events $\{J_n = j\}$ or $\{J_n \neq j\}$ under $H_n$,
\begin{align*}
    \doublesum\E{}{\mathbb{P}_{}\left(i_*(\theta) = i \cap  J_n = j \mid H_n\right)} &= \doublesum\E{}{\mathbb{P}_{}\left(i_*(\theta) = i \cap  J_n = j \mid H_n\right)(\I{J_n = j} + \I{J_n \neq j} ) }\\
    &=\doublesum\E{}{\mathbb{P}\left(i_*(\theta) = i \mid  J_n = j, H_n\right)\mid J_n = j }\mathbb{P}(J_n = j)\,.\\
    &=\doublesum \E{}{\mathbb{P}\left(i_*(\theta) = i \mid  J_n = j, H_n\right)\I{J_n = j}}\,,
\end{align*}
where the second equality holds by tower rule conditionally on $H_n$, and the last equation holds because $J_n$ is deterministic conditionally on $H_n$. Overall, this gives
\begin{align*}
    \mathcal{P}_n  =\doublesum\E{}{\mathbb{P}\left(i_*(\theta) = i \mid  J_n = j, H_n\right)\I{J_n = j}}\,.
\end{align*}
By definition of $i_*(\theta)$ in the MAB setting and applying Hoeffding inequality for sub-Gaussian random variables,
\begin{align}
\label{eq:Hoeffding_2}
\mathbb{P}\left(\argmax_{k \in [K]}\theta_{k}=i \mid H_n, J_n = j \right) &\leq  \mathbb{P}\left( \theta_{i} \geq \theta_{j} \mid H_n, J_n = j \right)\nonumber\\
&= \mathbb{P}\left( (\theta_{i} -\theta_{j}) - (\hat\mu_{n,i} - \hat\mu_{n,j}) \geq - (\hat\mu_{n,i} - \hat\mu_{n,j}) \mid H_n, J_n = j \right)\nonumber\\
&\leq \exp\left( - \frac{(\hat\mu_{n,i} - \hat\mu_{n,j})^2}{2(\hat\sigma^2_{n,i} + \hat\sigma^2_{n,j})}\right)\,.
\end{align}
Therefore,
\begin{align}
\label{temp_check}
    \E{}{\mathbb{P}\left(i_*(\theta) = i \mid  J_n = j, H_n\right)\I{J_n = j}}  &\leq \E{}{\exp\left( - \frac{(\hat\mu_{n,i} - \hat\mu_{n,j})^2}{2(\hat\sigma^2_{n,i} + \hat\sigma^2_{n,j})}\right)}
\end{align}
We now want to compute this above expectation with respect to $H_n$.

First, we remark that because the scheduling of arms $\left(A_1, ..., A_n\right)$ is deterministic, the law of $H_n = \left(A_1, Y_{A_1},..., A_n, Y_{A_n} \right)$ is the law of $\left(Y_{A_1}, ..., Y_{A_n}\right)$. Denoting $\pi_{H_n}$ the marginal distribution of $H_n$,
\begin{align*}
    \pi_{H_n}\left(H_n\right) &= \pi_{H_n}\left(Y_{A_1}, ..., Y_{A_n}\right) = \int_{\theta} \mathcal{L}_{\theta}(Y_{A_1}, ... Y_{A_n})P_0(\theta)\dint\theta\,,
\end{align*}
where $\mathcal{L}_{\theta}(Y_{A_1}, ... Y_{A_n})$ denotes the likelihood of $(Y_{A_1}, ... Y_{A_n})$ given parameter $\theta$ and $P_{0}(\theta) = \prod_{i \in [K]} P_{0, i}(\theta_i)$ since each mean reward $\theta_i$ is drawn independently from $P_{0, i}$ in the MAB setting. Since rewards given parameter $\theta$ are independent and identically distributed,
\begin{align}
\label{eq:convolution_simple_gaussian}
    \pi_{H_n}\left(H_n\right) &= \int_{\theta} \prod_{i\in[K]} \mathcal{L}_{\theta_i}\big((Y_t)_{t \in \mathcal{T}_i}\big) P_{0, i}(\theta_i)\dint\theta_i\nonumber \\
    &=  \int_{\theta} \prod_{i\in[K]} \cN\big((Y_t)_{t \in \mathcal{T}_i};\theta_i\bm{1}_{\omega_i n}, \sigma^2I_{\omega_i n}\big) \cN\big(\theta_i ;\mu_{0, i}, \sigma^2_{0, i}\big) \dint\theta_i\,,
\end{align}
where $\bm{1}_q$ denotes the vector of size $q$ whose all components are $1$s. 

\eqref{eq:convolution_simple_gaussian} is a convolution of Gaussians and can be computed easily \citep{Bishop2006},
\begin{align*}
\cN\big((Y_t)_{t \in \mathcal{T}_i};\theta_i\bm{1}_{\omega_i n}, \sigma^2I_{\omega_i n}\big) \cN\big(\theta_i ;\mu_{0, i}, \sigma^2_{0, i}\big) &= \cN\big( \left(Y_{t}\right)_{t \in \mathcal{A}_i} ; \mu_{0, i}\bm{1}_{\omega_i n} , \sigma^2 I_{\omega_i n} + \sigma^2_{0, i}\bm{1}_{\omega_i n}\bm{1}_{\omega_i n}^\top \big) \,.
\end{align*}
The above covariance matrix exhibits $\sigma^2 + \sigma^2_{0, i}$ on the diagonal and $\sigma^2_{0, i}$ out of diagonal.

We are now ready to compute some useful statistics : for any $i \in [K]$,
\begin{align}
 &\E{}{\hat\mu_{n,i}} = \E{}{\frac{\sigma^2}{\sigma^2 + \sigma^2_{0, i}\omega_i n}\mu_{0, i} + \frac{\sigma^2_{0, i}}{\sigma^2 + \sigma^2_{0, i}\omega_i n}\sum_{t \in \mathcal{T}_i}Y_{t}} = \mu_{0, i}\label{eq:standard_expected_posterior_mean}\\
 &\mathbb{V}\left(\hat\mu_{n,i}\right) = \frac{\sigma^4_{0, i}}{\left(\sigma^2 + \sigma^2_{0, i}\omega_i n\right)^2}\mathbb{V}\left(\sum_{t\in \mathcal{T}_i}Y_{t}\right) = \frac{\sigma^4_{0, i}}{\left(\sigma^2 + \sigma^2_{0, i}\omega_i n\right)}\omega_i n\,\label{eq:standard_variance_posterior_mean}\\
 &\E{}{\frac{1}{\omega_i n}\sum_{t \in \mathcal{T}_i}Y_{t}} = \mu_{0, i} \label{eq:standard_expected_empirical_mean} \\
 &\mathbb{V}\left(\frac{1}{\omega_i n}\sum_{t \in \mathcal{T}_i}Y_{t}\right) = \frac{1}{\omega_i^2 n^2}\left( \omega_i n \left(\sigma^2 + \sigma^2_{0, i}  \right)+ \left(\omega_i^2 n^2 - \omega_i n \right)\sigma^2_{0, i}\right)= \frac{\sigma^2}{\omega_i n} + \sigma^2_{0, i}\label{eq:standard_variance_empirical_mean}
\end{align}
Applying \cref{lemma:integration_gaussian} on \eqref{temp_check} and simplifying terms gives 
\begin{align*}
    \E{}{\mathbb{P}\left(i_*(\theta) = i \mid  J_n = j, H_n\right)\I{J_n = j}}  \leq \E{}{\exp\left( - \frac{(\hat\mu_{n,i} - \hat\mu_{n,j})^2}{2(\hat\sigma^2_{n,i} + \hat\sigma^2_{n,j})}\right)}=\frac{e^{ -\frac{(\mu_{0, i} - \mu_{0, j})^2}{2(\sigma^2_{0, i} + \sigma^2_{0, j})}}}{\sqrt{1 + n \phi_{i,j}}}\,.
\end{align*}

\subsection{Proofs for Linear Bandits}
\label{subsec:app_proof_linear_bandit}

\begin{theorem}[Complete statement of \cref{th:linear_model_upper_bound}]
Assume that $x_i \neq x_j$ for any $i \neq j$. Then, for all $\omega \in \Delta^+_K$, the expected PoE of \emph{$\AdaBAI$} using allocation $\omega$ under linear bandit problem \eqref{eq:linear_gaussian} is upper bounded as
\begin{align*}
  \mathcal{P}_n &\leq \doublesum \frac{1}{\sqrt{1 + \frac{ c_{i,j}(\omega) }{ \lVert x_i - x_j\rVert^2_{\hat\Sigma_n} }}} e^{ -\frac{(\mu_0^\top x_i - \mu_0^\top x_j)^2}{2\lVert x_i - x_j\rVert^2_{\hat\Sigma_0}}} \,,
\end{align*}
where:
\begin{align}\label{eq:linear_c_i_j}
    &\Cova(\hat\mu_n) = \frac{1}{\sigma^4} \hat\Sigma_n \left( \underbrace{\sum_{i \in [K]} \left(\omega_i n (\sigma^2 + x_i^\top \Sigma_0 x_i) \right) x_i x_i^\top}_{\text{variance terms}} + \underbrace{\sum_{i \in [K]}\sum_{j \in [K]\setminus \{j\}} x_i^\top \Sigma_0 x_j \omega_i \omega_j n^2 x_i x_j^\top}_{\text{covariance between arms}} \right) \hat\Sigma_n\nonumber \\
    &c_{i, j}=\lVert x_i - x_j \rVert^2_{ \Cova(\hat\mu_n) }\,.
\end{align}
\end{theorem}

\textbf{Proof of \cref{th:linear_model_upper_bound}. }

The proof for the linear model follows the same steps as the MAB model by rewriting $\mathcal{P}_n$ as
\begin{align*}
    \mathcal{P}_n  =\doublesum\E{}{\mathbb{P}\left(i_*(\theta) = i \mid  J_n = j, H_n\right)\I{J_n = j}}\,.
\end{align*}
By definition of $i_*(\theta)$ and $J_n$ in the linear bandit setting,
\begin{align*}
    \mathbb{P}\left(i_*(\theta) = i \mid H_n, J_n = j\right) &= \mathbb{P}\left(\forall k\in [K],\, \theta^\top x_i \geq \theta^\top x_k \mid H_n, J_n = j\right)  \\
    &\leq \mathbb{P}\left(\theta^\top x_i \geq \theta^\top x_j \mid J_n = j, H_n\right)\\
    &\leq \exp\left( -\frac{\lVert \hat\mu_n\rVert^2_{(x_i - x_j)(x_i - x_j)^\top}}{2 \lVert x_i - x_j \rVert^2_{\hat\Sigma_n}} \right)\,,
\end{align*}
where the last inequality follows from Hoeffding inequality for sub-Gaussian random variables. Taking the expectation with respect to $H_n$,
\begin{align}
\label{eq:temp_expectation_linear}
    \E{}{\mathbb{P}\left(i_*(\theta) = i \mid H_n, J_n = j\right)\I{J_n = j}} &\leq \E{}{\exp\left( -\frac{\lVert \hat\mu_n\rVert^2_{(x_i - x_j)(x_i - x_j)^\top}}{2 \lVert x_i - x_j \rVert^2_{\hat\Sigma_n}} \right)}\,.
\end{align}
Then we remark that the expectation of $\hat\mu_n$ with respect to $H_n$ is
\begin{align*}
    \E{}{\hat\mu_n} = \E{}{\hat\Sigma_n \left(\Sigma_0^{-1}\mu_0 + \frac{1}{\sigma^2}\sum_{t\in [n]}Y_t x_{A_t}\right)} = \hat\Sigma_n \left(\Sigma_0^{-1}\mu_0 + \frac{1}{\sigma^2} \E{}{\sum_{t\in [n]}Y_t x_{A_t}}\right)\,,
\end{align*}
since the scheduling $(A_1, \dots A_n)$ is known beforehand. Now,
\begin{align*}
    \E{}{\sum_{t\in [n]}Y_t x_{A_t}} = \sum_{t \in [n]} \E{}{Y_t}x_{A_t} = \sum_{t \in [n]}\mu_0^\top x_{A_t} x_{A_t}\,,
\end{align*}
where $\E{}{Y_t}$ was obtained by marginalizing the likelihood over the prior distribution as in \eqref{eq:convolution_simple_gaussian}.

Rearranging the terms permits to conclude that $\E{}{\hat\mu_n} = \mu_0$. Then we can compute the expectation in \eqref{eq:temp_expectation_linear} by applying \cref{lemma:integration_gaussian}, Sylvester identity, and some simplifications:
\begin{align}
\label{eq:temp_justification_Lambda_ij}
    \E{}{\exp\left( -\frac{\lVert \hat\mu_n\rVert^2_{(x_i - x_j)(x_i - x_j)^\top}}{2 \lVert x_i - x_j \rVert^2_{\hat\Sigma_n}} \right)} &= \frac{1}{\sqrt{1 + \frac{ \lVert x_i - x_j \rVert^2_{ \Cova(\hat\mu_n) } }{ \lVert x_i - x_j\rVert^2_{\hat\Sigma_n} }}} e^{ -\frac{1}{2}\lVert \mu_0 \rVert^2_{\Lambda_{ij}} } \,. 
\end{align}
where from \cref{lemma:integration_gaussian},
\begin{align*}
    \Lambda_{i,j} &= \Cova(\hat\mu_n)^{-1} - \Cova(\hat\mu_n)^{-1}\left(\Cova(\hat\mu_n)^{-1} + \frac{(A_i - A_j)(A_i - A_j)}{\lVert A_i - A_j\rVert^2_{\hat\Sigma_n}}\right)^{-1}\Cova(\hat\mu_n)^{-1}\\
    &= \Cova(\hat\mu_n)^{-1} -\Cova(\hat\mu_n)^{-1}\left(I_d + \Cova(\hat\mu_n)\frac{(x_i - x_j)(x_i - x_j)}{\lVert x_i - x_j\rVert^2_{\hat\Sigma_n}}\right)^{-1}\\
    &= \frac{(x_i - x_j)(x_i - x_j)^\top}{\lVert x_i - x_j \rVert^2_{\hat\Sigma_n}+\lVert x_i - x_j \rVert^2_{ \Cova(\hat\mu_n)}}\,.
\end{align*}
The last equality follows from an application of Sherman-Morrison identity.
Applying the law of total expectation,
\begin{align*}
    \Cova(\theta) = \E{}{\Cova(\theta \mid H_n)} + \Cova(\E{}{\theta \mid H_n}) = \hat\Sigma_n + \Cova(\hat\mu_n)\,.
\end{align*}
Therefore,
\begin{align*}
    \lVert x_i - x_j \rVert^2_{\hat\Sigma_n}+\lVert x_i - x_j \rVert^2_{ \Cova(\hat\mu_n)} = \lVert x_i - x_j \rVert^2_{\hat\Sigma_n +  \Cova(\hat\mu_n)}= \lVert x_i - x_j \rVert^2_{\Sigma_0}\,.
\end{align*}
Plugging these into \eqref{eq:temp_justification_Lambda_ij}, we obtain
\begin{align*}
    \mathbb{E}\left[\mathbb{P}\left(i_*(\theta) = i \mid H_n, J_n = j\right)\right] \leq \frac{e^{ -\frac{\lVert\mu_0 \rVert^2_{(x_i - x_j)(x_i - x_j)^\top}}{2\lVert x_i - x_j\rVert^2_{\hat\Sigma_0}}}}{\sqrt{1 + \frac{ \lVert x_i - x_j \rVert^2_{ \Cova(\hat\mu_n) } }{ \lVert x_i - x_j\rVert^2_{\hat\Sigma_n} }}} = \frac{e^{ -\frac{(\mu_0^\top x_i - \mu_0^\top x_j)^2}{2\lVert x_i - x_j\rVert^2_{\hat\Sigma_0}}}}{\sqrt{1 + \frac{ \lVert x_i - x_j \rVert^2_{ \Cova(\hat\mu_n) } }{ \lVert x_i - x_j\rVert^2_{\hat\Sigma_n} }}} \,.
\end{align*}

\textbf{Computation of $\boldsymbol{\Cova(\hat\mu_n)}$.}

By definition of Gaussian posteriors in linear bandit in \eqref{eq:linear_gaussian_posteriors},
\begin{align*}
    \Cova(\hat\mu_n) &= \Cova\left(\hat\Sigma_n \left(\Sigma_0^{-1}\mu_0 + \frac{1}{\sigma^2}B_n\right) \right) = \hat\Sigma_n \Cova\left(\Sigma_0^{-1}\mu_0 + \frac{1}{\sigma^2}B_n\right)\hat\Sigma_n =\frac{1}{\sigma^4}\hat\Sigma_n \Cova( B_n) \hat\Sigma_n\,,
\end{align*}
and
\begin{align*}
    \Cova(B_n) &= \sum_{t \in [n]}\mathbb{V}(Y_t x_{A_t})+ \sum_{t \in [n]}\sum_{t' \in [n], t\neq t'}\Cova (Y_t x_{A_t}, Y_{t'}x_{A_{t'}})\\
    &= \sum_{t \in [n]}\mathbb{V}(Y_t) x_{A_t} x_{A_t}^\top + \sum_{t \in [n]}\sum_{t' \in [n], t\neq t'}\Cova(Y_t, Y_{t'}) x_{A_t} x_{A_{t'}}^\top\\
    &= \sum_{k \in [K]} \omega_k n \mathbb{V}(Y_{x_k})x_k x_k^\top + \doublesum\omega_i \omega_j n^2 \Cova(Y_{x_i}, Y_{x_j})x_i x_j^\top\\
    &= \sum_{k \in [K]} \omega_k n (\sigma^2 + x_k^\top \Sigma_0 x_k))x_k x_k^\top + \doublesum\omega_i \omega_j n^2 (x_i^\top \Sigma_0 x_j) x_i x_j^\top\,.
\end{align*}

\textbf{Proof of \cref{cor:linear_with_optimal_design}.}

We first prove a useful lemma that holds for Bayesian G-optimal design.
\begin{lemma}
    \label{lemma:G_optimal_design_uuseful}
    Let $\cX$ a finite set such that $|\cX| = K$, $\xi : \cX\to [0, 1]$ a distribution on $\cX$ so that $\sum_{x \in \cX}\xi(x) = 1$, $V_n(\xi) = \sum_{x \in \cX} \xi(x) x x^\top + \frac{\sigma^2}{n}\Sigma_0^{-1}$, $\Sigma_0 \in \real^{d\times d}$ a diagonal matrix, and $f(\xi) = \log\det\big(V_n(\xi)\big)$.
    If $\xi^*=\argmin_{\xi \in \Delta_K} f(\xi)$, then $\max_{x\in \cX}\lVert x \rVert^2_{V_n(\xi)^{-1}}\leq d$.
\end{lemma}

\begin{proof}[Proof of \cref{lemma:G_optimal_design_uuseful}]
By concavity of $\xi \mapsto f(\xi)$, we have for any $\xi$ that
\begin{align*}
    0 &\geq \langle \nabla f(\xi^*), \xi - \xi^*
   \rangle = \sum_{x \in \cX}\xi(x)\big[ \nabla f(\xi^*) \big]_x - \langle\xi^*,\nabla f(\xi^*) \rangle\,,
\end{align*}
and since this holds for any pdf $\xi$, choosing $\xi=\delta_{x'}$ for an arbitrary action $x'$ yields 
\begin{align*}
\big[ \nabla f(\xi^*) \big]_{x'} \leq \langle\xi^*,\nabla f(\xi^*) \rangle \quad \text{ for any }x'\in \cX\,.
\end{align*}
Since r.h.s. does not depend on $x'$,
\begin{align}
\label{eq:temp_dirac}
\max_{x\in\cX}\big[ \nabla f(\xi^*) \big]_{x} \leq \langle\xi^*,\nabla f(\xi^*) \rangle \,.
\end{align}

By the property of the gradient of log-determinant, $\big[ \nabla f(\xi^*) \big]_{x} = \lVert x\rVert_{V_n(\xi^*)^{-1}}$. Therefore, for any $\xi$,
\begin{align*}
    \langle \xi, \nabla f(\xi)\rangle &= \sum_{x \in \cX} \xi(x)\lVert x \rVert^2_{V^{-1}(\xi)}\\
    &= \sum_{x \in \cX} \xi(x)x^\top V_n(\xi)^{-1}x\\
    &= \mathrm{Tr}\left(\sum_{x \in \cX} \xi(x)x^\top V_n(\xi)^{-1}x\right)\\
    &= \mathrm{Tr}\left(\sum_{x \in \cX} \xi(x) x x^\top \left(\sum_{x' \in \cX } \xi(x')x'x'^\top + \frac{\sigma^2}{n}\Sigma_0^{-1}\right)^{-1}\right)\\
    &= \mathrm{Tr}\left( E\left(E + \frac{\sigma^2}{n}\Sigma_0^{-1} \right)^{-1} \right)\qquad \qquad \qquad \qquad \text{ where }E = \sum_{x \in \cX}\xi(x)x x^\top\\
    &= \mathrm{Tr}\left(\left(I_d + \frac{\sigma^2}{n}\Sigma_0^{-1} E^{-1}\right)^{-1}\right)\\
    &= \mathrm{Tr}\left(I_d - \frac{\sigma^2}{n}\Sigma_0^{-1} \left(I_d + \frac{\sigma^2}{n}E^{-1}\Sigma_0^{-1}\right)^{-1}E^{-1}\right)\qquad \text{(Woodburry identity)}\\
    &=  \mathrm{Tr}(I_d) - \frac{\sigma^2}{n}\mathrm{Tr}\left(\left(E\Sigma_0 + \frac{\sigma^2}{n}I_d\right)^{-1}\right)\\
    &\leq  \mathrm{Tr}(I_d) = d\,.
\end{align*}
All putting together in \eqref{eq:temp_dirac} with $\xi^* = \argmin_{\xi \in \Delta_K}f(\xi)$ implies $\max_{x \in \cX} \lVert x\rVert_{V(\xi^*)^{-1}}\leq d$.
\end{proof}
A direct implication of \cref{lemma:G_optimal_design_uuseful} is that $\max_{x\in \cX} \lVert x\rVert_{\hat\Sigma_n}\leq \frac{d\sigma^2}{n}$. Therefore,
\begin{align*}
    \mathcal{P}_n \leq \doublesum \frac{e^{ -\frac{(\mu_0^\top x_i - \mu_0^\top x_j)^2}{2\lVert x_i - x_j\rVert^2_{\hat\Sigma_0}}}}{\sqrt{1 + \frac{ c_{i, j} }{ \lVert x_i - x_j\rVert^2_{\hat\Sigma_n} }}} \leq \doublesum \frac{e^{ -\frac{(\mu_0^\top x_i - \mu_0^\top x_j)^2}{2\lVert x_i - x_j\rVert^2_{\hat\Sigma_0}}}}{\sqrt{1 + \frac{ c_{i, j} }{ 2\max_{x \in \cX}\lVert x\rVert^2_{\hat\Sigma_n} }}} \leq \doublesum \frac{e^{ -\frac{(\mu_0^\top x_i - \mu_0^\top x_j)^2}{2\lVert x_i - x_j\rVert^2_{\hat\Sigma_0}}}}{\sqrt{1 + n\frac{ c_{i, j} }{ 2d\sigma^2 }}} \,.
\end{align*}
\subsection{Proofs for Hierarchical Bandits}
\label{subsec:app_proof_hierarchical}
We begin by stating the complete proof.

\begin{theorem}[Complete statement of \cref{th:upper_bound_hierarchical_model}]
\label{th:complete_upper_bound_hierarchical_model}
    For all $\omega \in \Delta^+_K$, the expected PoE of \emph{$\AdaBAI$} using allocation $\omega$ under the hierarchical bandit problem \eqref{eq:model_gaussian} is upper bounded as
    \begin{align*}
    \mathcal{P}_n \leq \doublesum \frac{1}{\sqrt{1 + \frac{ c_{i,j} }{\hat\sigma_{n,i}+\hat\sigma_{n,j}}}}e^{-\frac{( \nu^\top b_i - \nu^\top b_j )^2}{2( \normw{b_i - b_j}{\Sigma}^2 + \sigma^2_{0, i}+\sigma^2_{0, j})}}\,,
    \end{align*}
where    
\begin{tiny} 
    \begin{align}\label{eq:hierarchical_c_i_j}
    &c_{i, j} = \mathbb{V}(\hat \mu_{n,i} - \hat\mu_{n,j}) = \frac{\sigma^4_{0, i}}{(\sigma^2_{0, i}\omega_i n + \sigma^2)^2}\mathbb{V}(B_{n,i}) + \frac{\sigma^4_{0, j}}{(\sigma^2_{0, j}\omega_j n + \sigma^2)^2}\mathbb{V}(B_{n,j})\\
    &+\frac{\sigma^4}{(\sigma^2_{0, i}\omega_i n + \sigma^2)^2}\sum_{k \in [K]}\frac{(b_k^\top \breve{\Sigma}_nb_i)^2}{(\sigma^2 + \omega_k n\sigma^2_{0, k})^2}\mathbb{V}(B_{n,k}) +  \frac{\sigma^4}{(\sigma^2_{0, j}\omega_j n + \sigma^2)^2}\sum_{k \in [K]}\frac{(b_k^\top \breve{\Sigma}_nb_j)^2}{(\sigma^2 + \omega_kn\sigma^2_{0, k})^2}\mathbb{V}(B_{n,k})\nonumber\\
    &+ \frac{\sigma^4}{(\sigma^2_{0, i}\omega_i n + \sigma^2)^2}\sum_{k \in [K]}\sum_{k^\prime \in [K]\setminus\{k\}} \frac{(b_k^\top \breve{\Sigma}_nb_i).(b_{k^\prime}^\top \breve{\Sigma}_nb_i)}{(\sigma^2_{0, k}\omega_k n + \sigma^2).(\sigma^2_{0, k^\prime}\omega_{k^\prime} n + \sigma^2)}\cov(B_{n,k}, B_{n,k^\prime}) \nonumber\\
    &+ \frac{\sigma^4}{(\sigma^2_{0, j}\omega_j n + \sigma^2)^2}\sum_{k \in [K]}\sum_{k^\prime \in [K]\setminus\{k\}} \frac{(b_k^\top \breve{\Sigma}_nb_j).(b_{k^\prime}^\top \breve{\Sigma}_nb_j)}{(\sigma^2_{0, k}\omega_k n + \sigma^2).(\sigma^2_{0, k^\prime}\omega_{k^\prime} n + \sigma^2)}\cov(B_{n,k}, B_{n,k^\prime})\nonumber\\
    &- \frac{2\sigma^4}{(\sigma^2_{0, i}\omega_i n + \sigma^2)(\sigma^2_{0, j}\omega_j n + \sigma^2)} \left[\sum_{k \in [K]}\frac{(b_k^\top \breve{\Sigma}_n b_i)(b_k^\top \breve{\Sigma}_nb_j)}{(\sigma^2 + \omega_k n \sigma^2_{0, k})^2}\mathbb{V}(B_{n,k}) + \sum_{k \in [K]}\sum_{k^\prime \in [K]\setminus\{k\}} \frac{(b_k^\top \breve{\Sigma}_nb_i).(b_{k^\prime}^\top \breve{\Sigma}_nb_j)}{(\sigma^2_{0, k}\omega_k n + \sigma^2).(\sigma^2_{0, k^\prime}\omega_{k^\prime} n + \sigma^2)}\cov(B_{n,k}, B_{n,k^\prime})\right]\nonumber\\
    &- \frac{2\sigma^2_{0, i}\sigma^2_{0, j}}{(\sigma^2_{0, i}\omega_i n + \sigma^2)(\sigma^2_{0, j}\omega_j n + \sigma^2)}\cov(B_{n,i}, B_{n,j}) \nonumber\\
    &+ \frac{2\sigma^2\sigma^2_{0, i}}{(\sigma^2_{0, i}\omega_i n + \sigma^2)^2}\left[\sum_{k \in [K]\setminus \{i\}} \frac{b_k^\top \breve{\Sigma}_nb_i}{(\sigma^2_{0, k}\omega_k n + \sigma^2)}\cov(B_{n,k}, B_{n,i}) + \frac{b_i^\top \breve{\Sigma}_nb_i}{(\sigma^2_{0, i}\omega_i n + \sigma^2)}\mathbb{V}(B_{n,i})\right]\nonumber\\
    &+ \frac{2\sigma^2\sigma^2_{0, j}}{(\sigma^2_{0, j}\omega_j n + \sigma^2)^2}\left[\sum_{k \in [K]\setminus \{j\}} \frac{b_k^\top \breve{\Sigma}_nb_j}{(\sigma^2_{0, k}\omega_k n + \sigma^2)}\cov(B_{n,k}, B_{n,j}) + \frac{b_j^\top \breve{\Sigma}_nb_i}{(\sigma^2_{0, j}\omega_j n + \sigma^2)}\mathbb{V}(B_{n,j})\right]\nonumber\\
    &- \frac{2\sigma^2 \sigma^2_{0, j}}{(\sigma^2_{0, i}\omega_i n + \sigma^2)(\sigma^2_{0, j}\omega_j n + \sigma^2)}\left[\sum_{k \in [K]\setminus\{j\}}\frac{b_k^\top \breve{\Sigma}_nb_i}{\sigma^2 + \omega_k n \sigma^2_{0, k}}\cov(B_{n,k}, B_{n,i}) + \frac{b_j^\top \breve{\Sigma}_nb_i}{\sigma^2 + \omega_jn\sigma^2_{0, j}}\mathbb{V}(B_{n,j})\right]\nonumber\\
    &- \frac{2\sigma^2 \sigma^2_{0, i}}{(\sigma^2_{0, j}\omega_j n + \sigma^2)(\sigma^2_{0, i}\omega_i n + \sigma^2)}\left[\sum_{k \in [K]\setminus\{i\}}\frac{b_k^\top \breve{\Sigma}_nb_j}{\sigma^2 + \omega_k n \sigma^2_{0, k}}\cov(B_{n,k}, B_{n,j}) + \frac{b_i^\top \breve{\Sigma}_nb_j}{\sigma^2 + \omega_i n\sigma^2_{0, i}}\mathbb{V}(B_{n,i})\right]\,,\nonumber
\end{align}
\end{tiny}

where we defined $\breve\Sigma_n$ from \eqref{eq:cluster_posterior},
\begin{align*}
    &\breve{\Sigma}_n^{-1} = \Sigma^{-1} + \sum_{k \in [K]}\frac{\omega_k n}{\sigma^2 + \omega_k n \sigma^2_{0, k}}b_k b_k^\top\,, &\mathbb{V}(B_{n,k}) =  \omega_k n \sigma^2 + \omega_k^2 n^2 (\sigma^2_{0, k} +b_k^\top \Sigma b_k)\\
    &\cov(B_{n,k}, B_{n,k^\prime}) = \omega_k \omega_{k^\prime}n^2 b_k^\top \Sigma b_{k^\prime}\,.
\end{align*}
\end{theorem}

\begin{proof}[Proof of \cref{th:complete_upper_bound_hierarchical_model}]
This proof follows the same idea of the proof of \cref{th:upper bound}. We first write $\mathcal{P}_n$ as
\begin{align*}
    \mathcal{P}_n  =\doublesum\E{}{\mathbb{P}\left(i_*(\theta) = i \mid  J_n = j, H_n\right)\I{J_n = j}}\,.
\end{align*}
Following \eqref{eq:Hoeffding_2}, by applying Hoeffding inequality for sub-Gaussian random variables,
\begin{align*}
\mathbb{P}\left(\argmax_{k \in [K]}\theta_{k}=i \mid H_n, J_n = j \right) &\leq  \mathbb{P}\left( \theta_{i} \geq \theta_{j} \mid H_n, J_n = j \right)\nonumber\\
&= \mathbb{P}\left( (\theta_{i} -\theta_{j}) - (\hat\mu_{n,i} - \hat\mu_{n,j}) \geq - (\hat\mu_{n,i} - \hat\mu_{n,j}) \mid H_n, J_n = j \right)\nonumber\\
&\leq \exp\left( - \frac{(\hat\mu_{n,i} - \hat\mu_{n,j})^2}{2(\hat\sigma^2_{n,i} + \hat\sigma^2_{n,j})}\right)\,,
\end{align*}
where $\hat\mu_{n,i}$ and $\hat\sigma^2_{n,i}$ are given by \eqref{eq:posterior}. Taking the expectation with respect to the history $H_n$,
\begin{align*}
  \E{}{\mathbb{P}(i_*(\theta)=i \mid J_n = j, H_n) \I{J_n = j}}\leq \E{}{e^{-\frac{(\hat\mu_{n,j} - \hat\mu_{n,i})^2}{2(\hat\sigma^2_{n,i} + \hat\sigma^2_{n,j})}}}\,.
\end{align*}
Since $\hat\mu_{n,i} - \hat\mu_{n,j} \sim \cN\big( \E{}{\hat\mu_{n,i}}-\E{}{\hat\mu_{n,j}}, \mathbb{V}(\hat\mu_{n,i} - \hat\mu_{n,j})\big)$, applying \cref{lemma:integration_gaussian} gives
\begin{align*}
    \E{}{e^{-\frac{(\hat\mu_{n,j} - \hat\mu_{n,i})^2}{2(\hat\sigma^2_{n,i} + \hat\sigma^2_{n,j})}}} &= \frac{1}{\sqrt{1 + \frac{\mathbb{V}(\hat\mu_{n,i} - \hat\mu_{n,j})}{\hat\sigma^2_{n,i} + \hat\sigma^2_{n,j}}}}\exp\left(-\frac{(\E{}{\hat\mu_{n,i}} - \E{}{\hat\mu_{n,j}})^2}{2(\hat\sigma^2_{n,i} + \hat\sigma^2_{n,j})}\frac{1}{1 + \frac{\mathbb{V}(\hat\mu_{n,i} - \hat\mu_{n,j})}{\hat\sigma^2_{n,i} + \hat\sigma^2_{n,j}}}\right)
\end{align*}
Therefore,
\begin{align*}
    \E{}{\mathbb{P}(i_*(\theta)=i \mid J_n = j, H_n)}\leq \frac{1}{\sqrt{1 + \frac{\mathbb{V}(\hat\mu_{n,i} - \hat\mu_{n,j})}{\hat\sigma^2_{n,i} + \hat\sigma^2_{n,j}}}}e^{-\frac{\left(\E{}{\hat\mu_{n,i}} - \E{}{\hat\mu_{n,j}}\right)^2}{2\left(\hat\sigma^2_{n,i} + \hat\sigma^2_{n,j} + \mathbb{V}(\hat\mu_{n,i} - \hat\mu_{n,j}) \right)}}\,.
\end{align*}
Now we want to simplify $\hat\sigma^2_{n,i} + \hat\sigma^2_{n,j} + \mathbb{V}(\hat\mu_{n,i} - \hat\mu_{n,j})$. on one hand, by the law of total variance,
\begin{align*}
    \mathbb{V}(\theta_i - \theta_j) &= \E{}{\mathbb{V}(\theta_i - \theta_j \mid H_n)} + \mathbb{V}(\E{}{\theta_i - \theta_j \mid H_n})= \hat\sigma^2_{n, i}+ \hat\sigma^2_{n, j}+\mathbb{V}(\hat\mu_{n, i}-\hat\mu_{n, j})\,,
\end{align*}
On the other hand,
\begin{align*}
    \mathbb{V}(\theta_i - \theta_j) = \E{}{\mathbb{V}(\theta_i - \theta_j \mid \mu)} + \mathbb{V}(\E{}{\theta_i - \theta_j \mid \mu}) &= \sigma^2_{0, i} + \sigma^2_{0, j} + \mathbb{V}( (b_i - b_j)^\top \mu )\\
    &= \sigma^2_{0, i} + \sigma^2_{0, j} + \normw{b_i - b_j}{\Sigma}^2\,.
\end{align*}
Combining these two last equations gives $\hat\sigma^2_{n, i}+ \hat\sigma^2_{n, j}+\mathbb{V}(\hat\mu_{n, i}-\hat\mu_{n, j}) = \sigma^2_{0, i} + \sigma^2_{0, j} + \normw{b_i - b_j}{\Sigma}^2\,$.

Therefore,
\begin{align}\label{eq:temp eq 1 hierarchical model}
    \E{}{\mathbb{P}(i_*(\theta)=i \mid J_n = j, H_n)}\leq \frac{1}{\sqrt{1 + \frac{\mathbb{V}(\hat\mu_{n,i} - \hat\mu_{n,j})}{\hat\sigma^2_{n,i} + \hat\sigma^2_{n,j}}}}e^{-\frac{\left(\E{}{\hat\mu_{n,i}} - \E{}{\hat\mu_{n,j}}\right)^2}{2\left(\sigma^2_{0, i} + \sigma^2_{0, j} + \normw{b_i - b_j}{\Sigma}^2\right)}}\,.
\end{align}

\textbf{Computing} $\bm{\mathbb{V}(\hat\mu_{n,i} - \hat\mu_{n,j})}.$

The rest of the proof consists to compute $\E{}{\hat\mu_{n,i}}$ and $\mathbb{V}(\hat\mu_{n,i} - \hat\mu_{n,j})$ for $(i, j)$. Denoting $Q$ the latent prior distribution $\mu \sim Q$ and $\pi_{H_n}$ the law of $H_n$,
\begin{align*}
    \pi_{H_n}(H_n) &= \pi_{H_n}(Y_{A_1}, ..., Y_{A_n})\\
    &= \iint_{(\theta, \mu)}\mathcal{L}_{\theta}(Y_{A_1}, ..., Y_{A_n})P_0(\theta\mid\mu)Q(\mu)\dint\theta \dint\mu\\
    &= \iint_{(\theta, \mu)} \prod_{i \in [K]}\mathcal{L}_{\theta_i}\left((Y_{t})_{t \in \mathcal{T}_i}\right)P_{0, i}(\theta_i\mid\mu)Q(\mu)\dint\theta_i \dint\mu\\
    &= \int_{\mu}\left[ \prod_{i\in [K]}\int_{\theta_i}\cN\left((Y_t)_{t \in \mathcal{T}_i}) ; \theta_i\bm{1}_{\omega_i n}, \sigma^2 I_{\omega_i n}\right)\cN(\theta_i ; b_i^\top \mu, \sigma^2_{0, i})\dint\theta_i \right]Q(\mu)\dint\mu\,.
\end{align*}
From properties of Gaussian convolutions \citep{Bishop2006},
\begin{align*}
    \int_{\theta_i}\cN\left((Y_{t})_{t \in \mathcal{T}_i}) ; \theta_i\bm{1}_{\omega_i n}, \sigma^2 I_{\omega_i n}\right)\cN(\theta_i ; b_i^\top \mu, \sigma^2_{0, i})\dint\theta_i  &= \cN\left((Y_{t})_{t \in \mathcal{T}_i}) ; (b_i^\top\mu)\bm{1}_{\omega_i n}, \sigma^2. I_{\omega_i n} + \bm{1}_{\omega_i n}\bm{1}_{\omega_i n}^\top \sigma^2_{0, i} \right)\,.
\end{align*}
Therefore,
\begin{align*}
    &\prod_{i \in [K]}\int_{\theta_i}\cN\left((Y_{t})_{t \in \mathcal{T}_i}) ; \theta_i\bm{1}_{\omega_i n}, \sigma^2 I_{\omega_i n}\right)\cN(\theta_i ; b_i^\top \mu, \sigma^2_{0, i})\dint\theta_i\\
    &= \cN(H_n ; \sum_{i \in [K]} e_i\left(\real^K\right) \otimes\left(\sum_{t\in [\omega_i n]} e_t\left(\real^{\omega_i n}\right) \otimes b_i^\top\right)\otimes \mu, \, I_K \otimes (\sigma^2 I_{\omega_i n} + \bm{1}_{\omega_i n}\bm{1}_{\omega_i n}^\top \sigma^2_{0, i}))\,,
\end{align*}
where we define explicitly $e_i\left(\real^K\right)$ as the $i^{th}$ base vector of $\real^K$.

Therefore,
\begin{align*}
    \pi(H_n) &= \int_{\mu}\cN(H_n ; \sum_{i \in [K]} e_i\left(\real^K\right) \otimes\left(\sum_{t\in \omega_i n}e_t\left(\real^{\omega_i n}\right) \otimes b_i^\top\right)\otimes \mu, \, I_K \otimes (\sigma^2 I_{\omega_i n} + \bm{1}_{\omega_i n}\bm{1}_{\omega_i n}^\top \sigma^2_{0, i}))\cN(\mu;\nu, \Sigma)\dint\mu\\
    &= \cN(H_n ; \mathring{\mu}, \mathring{\Sigma})\,,
\end{align*}
where $\mathring{\mu}\in \mathbb{R}^n$, $\mathring{\Sigma}\in \mathbb{R}^{n\times n}$ with
\begin{align}\label{eq:marginal posterior hierarchical}
    &\mathring{\mu} = \sum_{i \in [K]} e_i\left(\real^K\right) \otimes\left(\sum_{t\in \omega_i n}e_t\left(\real^{\omega_i n}\right) \otimes b_i^\top\right)\otimes \nu \nonumber\\
    &\mathring{\Sigma} = I_K \otimes (\sigma^2 I_{\omega_i n} + \bm{1}_{\omega_i n}\bm{1}_{\omega_i n}^\top \sigma^2_{0, i}) \nonumber\\
    &+ \left[\sum_{i \in [K]}e\left(\real^K\right) \otimes\left(\sum_{t\in [\omega_i n]}e_t\left(\real^{\omega_i n}\right) \otimes b_i^\top\right)\right] \Sigma \left[\sum_{i \in [K]}e_i\left(\real^K\right) \otimes\left(\sum_{t\in [\omega_i n]}e_t\left(\real^{\omega_i n}\right) \otimes b_i^\top\right)\right]^\top \,.
\end{align}
The covariance matrix $\mathring{\Sigma}$ seems complex but has a simple structure. The first term $I_K \otimes (\sigma^2 I_{\omega_i n} + \bm{1}_{\omega_i n}\bm{1}_{\omega_i n}^\top \sigma^2_{0, i})$ is the same as in the standard model. The remaining term accounts for the correlation between distinct arms $(i, j)$, and this correlation is of the form $b_i^\top \Sigma b_j$.

Now we are ready to compute $\E{}{\hat\mu_{n,k}}$ for any arm $k\in[K]$: from \eqref{eq:conditional_posterior} and \eqref{eq:posterior},
\begin{align*}
    \E{}{\hat\mu_{n,k}} = \E{}{\tilde{\sigma}^2_{n,k}\left(\frac{\breve\mu_n^\top b_k}{\sigma^2_{0, k}}+\frac{B_{n,k}}{\sigma^2}\right)} &=\E{}{ \frac{\sigma^2\sigma^2_{0, i}}{\sigma^2_{0, k}\omega_k n + \sigma^2}\left(\frac{\breve\mu_n^\top b_k}{\sigma^2_{0, k}}+\frac{B_{n,k}}{\sigma^2}\right)}\\
    &= \frac{\sigma^2}{\sigma^2_{0, k}\omega_k n + \sigma^2}\E{}{\breve\mu_n^\top b_k} + \frac{\sigma^2_{0, k}}{\sigma^2_{0, k}\omega_k n + \sigma^2}\E{}{B_{n,k}}\,.
\end{align*}
From \eqref{eq:cluster_posterior},
\begin{align*}
    \breve\mu_n^\top b_k = \left( \nu^\top \Sigma^{-1} + \sum_{i \in [K]}\frac{B_{n,i}}{\sigma^2 + \omega_i n \sigma^2_{0, i}}b_i^\top \right)\breve{\Sigma}_n b_k &= \nu^\top \Sigma^{-1}\breve{\Sigma}_n b_k + \sum_{i \in [K]}\frac{B_{n,i}}{\sigma^2 + \omega_i n \sigma^2_{0, i}}b_i^\top \breve{\Sigma}_n b_k\,.
\end{align*}
By linearity,
\begin{align*}
    \E{}{\hat\mu_{n,k}} &= \frac{\sigma^2}{\sigma^2_{0, k}\omega_k n +\sigma^2}\left(\nu^\top \Sigma^{-1}\breve{\Sigma}_nb_k + \sum_{i \in [K]}\frac{\E{}{B_{n,i}} }{\sigma^2 + \omega_i n \sigma^2_{0, i}}b_i^\top \breve{\Sigma}_nb_k\right) + \frac{\sigma^2_{0, k}}{\sigma^2_{0, k} \omega_k n + \sigma^2}\E{}{B_{n,k}}\,.
\end{align*}
From \cref{eq:marginal posterior hierarchical}, $\,\E{}{B_{n,i}} = \omega_i n \nu^\top b_i$. Therefore,
\begin{align}\label{eq:temp_covariance_diff}
    \E{}{\hat\mu_{n,k}}  &= \frac{\sigma^2}{\sigma^2_{0, k}\omega_k n +\sigma^2}\left(\nu^\top \Sigma^{-1}\breve{\Sigma}_nb_k + \sum_{i \in [K]}\frac{  \omega_i n \nu^\top b_i}{\sigma^2 + \omega_i n \sigma^2_{0, i}}b_i^\top \breve{\Sigma}_nb_k\right) + \frac{\sigma^2_{0, k}}{\sigma^2_{0, k} \omega_k n + \sigma^2}\omega_k n \nu^\top b_k= \nu^\top b_k
\end{align}
Now we are ready to compute $\mathbb{V}(\hat\mu_{n,i} - \hat\mu_{n,j})$ for any $(i, j)$. From \eqref{eq:temp_covariance_diff},
\begin{align*}
    &\hat\mu_{n,i} - \hat\mu_{n,j} = \frac{\sigma^2}{\sigma^2_{0, i}\omega_i n + \sigma^2}\breve\mu_n^\top b_i + \frac{\sigma^2_{0, i}}{\sigma^2_{0, i}\omega_i n + \sigma^2}B_{n,i} -  \frac{\sigma^2}{\sigma^2_{0, j}\omega_j n + \sigma^2}\breve\mu_n^\top b_j - \frac{\sigma^2_{0, j}}{\sigma^2_{0, i}\omega_j n + \sigma^2}B_{n,j}\\
    &= \underbrace{\frac{\sigma^2}{\sigma^2_{0, i}\omega_i n + \sigma^2}\nu^\top \Sigma^{-1}\breve{\Sigma}_nb_i - \frac{\sigma^2}{\sigma^2_{0, j}\omega_j n + \sigma^2}\nu^\top \Sigma^{-1}\breve{\Sigma}_nb_j}_{\text{does not depend on observations}}   \\
    &+\underbrace{\frac{\sigma^2}{\sigma^2_{0, i}\omega_i n + \sigma^2}.\sum_{k \in [K]}\frac{B_{n,k}}{\omega_k n \sigma^2_{0, k}+\sigma^2}b_k^\top \breve{\Sigma}_nb_i}_{\textcolor{purple}{(1)}} + \underbrace{\frac{-\sigma^2}{\sigma^2_{0, j}\omega_j n + \sigma^2}.\sum_{k \in [K]}\frac{B_{n,k}}{\omega_k n \sigma^2_{0, k}+\sigma^2}b_k^\top \breve{\Sigma}_nb_j}_{\textcolor{purple}{(2)}} \\
    &+\underbrace{\frac{\sigma^2_{0, j}}{\sigma^2_{0, i}\omega_i n + \sigma^2}B_{n,i}}_{\textcolor{purple}{(3)}}+ \underbrace{\frac{-\sigma^2_{0, j}}{\sigma^2_{0, j}\omega_j n + \sigma^2}B_{n,j}}_{\textcolor{purple}{(4)}}\,.
\end{align*}
Since \textcolor{purple}{(1)}, \textcolor{purple}{(2)},\textcolor{purple}{(3)} and \textcolor{purple}{(4)} are correlated,
\begin{align}
    &\mathbb{V}(\hat\mu_{n,i} + \hat\mu_{n,j}) = \mathbb{V}\left({\textcolor{purple}{(1)}} + {\textcolor{purple}{(2)}} + {\textcolor{purple}{(3)}} + {\textcolor{purple}{(4)}}\right) = \sum_{i=1}^4 \mathbb{V}(\textcolor{purple}{(i)}) + \sum_{i=1}^4 \sum_{j=1, j\neq i}^4 2\cov({\textcolor{purple}{(i)}}, {\textcolor{purple}{(j)}})\label{eq:variance_marginal_posterior}\,.
\end{align}
We now compute each term of \eqref{eq:variance_marginal_posterior}:
\begin{tiny}
\begin{align*}
    \mathbb{V}({\textcolor{purple}{(1)}}) &= \frac{\sigma^4}{(\sigma^2_{0, i}\omega_i n + \sigma^2)^2}\mathbb{V}\left( \sum_{k \in [K]}\frac{B_{n,k}}{\omega_k n \sigma^2_{0, k}+\sigma^2}b_k^\top \breve{\Sigma}_nb_i \right)\\
    &= \frac{\sigma^4}{(\sigma^2_{0, i}\omega_i n + \sigma^2)^2}\left(\sum_{k \in [K]}\frac{(b_k^\top \breve{\Sigma}_nb_i)^2}{(\sigma^2 + \omega_k n \sigma^2_{0, k})^2}\mathbb{V}(B_{n,k}) + \sum_{(k, k^\prime), k\neq k^\prime}\cov\left(\frac{B_{n,k}}{\sigma^2 + \omega_k n \sigma^2_{0, k}}b_k^\top \breve{\Sigma}_nb_i, \frac{B_{n,k^\prime}}{\sigma^2 + \omega_{k^\prime} n \sigma^2_{0, k^\prime}}b_{k^\prime}^\top \breve{\Sigma}_nb_i\right)\right)\\
    &= \frac{\sigma^4}{(\sigma^2_{0, i}\omega_i n + \sigma^2)^2}\left(\sum_{k \in [K]}\frac{(b_k^\top \breve{\Sigma}_nb_i)^2}{(\sigma^2 + \omega_k n \sigma^2_{0, k})^2}\mathbb{V}(B_{n,k}) + \sum_{(k, k^\prime), k\neq k^\prime}\frac{(b_k^\top \breve{\Sigma}_nb_i)(b_{k^\prime}^\top \breve{\Sigma}_nb_i)}{(\sigma^2 + \omega_k n \sigma^2_{0, k})(\sigma^2 + \omega_{k^\prime} n \sigma^2_{0, k^\prime})}\cov(B_{n,k}, B_{n,k^\prime})\right)\\
    \mathbb{V}(\textcolor{purple}{(2)})&= \frac{\sigma^4}{(\sigma^2_{0, j}\omega_j n + \sigma^2)^2}\left(\sum_{k \in [K]}\frac{(b_k^\top \breve{\Sigma}_nb_j)^2}{(\sigma^2 + \omega_k n \sigma^2_{0, k})^2}\mathbb{V}(B_{n,k}) + \sum_{(k, k^\prime), k\neq k^\prime}\frac{(b_k^\top \breve{\Sigma}_nb_i)(b_{k^\prime}^\top \breve{\Sigma}_nb_j)}{(\sigma^2 + \omega_k n \sigma^2_{0, k})(\sigma^2 + \omega_{k^\prime} n \sigma^2_{0, k^\prime})}\cov(B_{n,k}, B_{n, k^\prime})\right)\,,\\
    \mathbb{V}(\textcolor{purple}{(3)})&= \frac{\sigma^4_{0, i}}{(\sigma^2_{0, i}\omega_i n+\sigma^2)^2}\mathbb{V}(B_{n,i})\\
    \mathbb{V}(\textcolor{purple}{(4)})&= \frac{\sigma^4_{0, j}}{(\sigma^2_{0, j}\omega_j n+\sigma^2)^2}\mathbb{V}(B_{n,j})\,,
\end{align*}
\begin{align*}
    \cov( \textcolor{purple}{(1)}, \textcolor{purple}{(2)} ) &= -\cov\left( \frac{\sigma^2}{\sigma^2_{0, i}\omega_i n + \sigma^2}.\sum_{k \in [K]}\frac{B_{n,k}}{\omega_k n \sigma^2_{0, k}+\sigma^2}b_k^\top \breve{\Sigma}_nb_i, \frac{\sigma^2}{\sigma^2_{0, j}\omega_j n + \sigma^2}.\sum_{k \in [K]}\frac{B_{n,k}}{\omega_k n \sigma^2_{0, k}+\sigma^2}b_k^\top \breve{\Sigma}_nb_j \right)\\
    &= -\frac{\sigma^4}{(\sigma^2_{0, i}\omega_i n + \sigma^2)(\sigma^2_{0, j}\omega_j n + \sigma^2)}\cov\left( \sum_{k \in [K]}\frac{B_{n,k}}{\omega_k n \sigma^2_{0, k}+\sigma^2}b_k^\top \breve{\Sigma}_nb_i, \sum_{k \in [K]}\frac{B_{n,k}}{\omega_k n \sigma^2_{0, k}+\sigma^2}b_k^\top \breve{\Sigma}_n b_j  \right)\\
    &=  -\frac{\sigma^4}{(\sigma^2_{0, i}\omega_i n + \sigma^2)(\sigma^2_{0, j}\omega_j n + \sigma^2)}\bigg[ \sum_{k \in [K]}\frac{(b_k^\top \breve{\Sigma}_nb_i)(b_k^\top \breve{\Sigma}_nb_j)}{(\sigma^2 + \omega_k n \sigma^2_{0, k})}\mathbb{V}(B_{n,k}) \\
    &\qquad\qquad\qquad\qquad\qquad\qquad\qquad\qquad+ \sum_{(k, k^\prime), k\neq k^\prime}\frac{(b_k^\top \breve{\Sigma}_nb_i)(b_{k^\prime}^\top \breve{\Sigma}_nb_j)}{(\sigma^2 + \omega_k n \sigma^2_{0, k})(\sigma^2 + \omega_{k^\prime} n \sigma^2_{0, k^\prime})}\cov(B_{n,k}, B_{n,k^\prime})\bigg]\\
    \cov( \textcolor{purple}{(3)}, \textcolor{purple}{(4)} ) &= -\cov\left(\frac{\sigma^2_{0, i}}{\sigma^2_{0, i}\omega_i n + \sigma^2}B_{n,i}, \frac{\sigma^2_{0, j}}{\sigma^2_{0, j}\omega_j n + \sigma^2}B_{n,j} \right)\\
    &= -\frac{\sigma^2_{0, i}\sigma^2_{0, j}}{(\sigma^2_{0, i}\omega_i n + \sigma^2)(\sigma^2_{0, j}\omega_j n + \sigma^2)}\cov(B_{n,i}, B_{n,j})\\
    \cov( \textcolor{purple}{(1)}, \textcolor{purple}{(4)} ) &= -\cov\left(\frac{\sigma^2}{\sigma^2_{0, i}\omega_i n + \sigma^2}.\sum_{k \in [K]}\frac{B_{n,k}}{\omega_k n \sigma^2_{0, k}+\sigma^2}b_k^\top \breve{\Sigma}_nb_i, \frac{\sigma^2_{0, j}}{\sigma^2_{0, j}\omega_j n + \sigma^2}B_{n,j}\right)\\
     &= -\frac{\sigma^2 \sigma^2_{0, j}}{(\sigma^2_{0, i}\omega_i n + \sigma^2)(\sigma^2_{0, j}\omega_j n + \sigma^2)}\cov\left( \sum_{k \in [K]}\frac{B_{n,k}}{\sigma^2 + \omega_k n \sigma^2_{0, k}}b_k^\top \breve{\Sigma}_nb_i, B_{n,j} \right)\\
     &= -\frac{\sigma^2 \sigma^2_{0, j}}{(\sigma^2_{0, i}\omega_i n + \sigma^2)(\sigma^2_{0, j}\omega_j n + \sigma^2)}\left[\sum_{k \in [K]\setminus\{j\}} \frac{b_k^\top \breve{\Sigma}_nb_i}{\sigma^2_{0, k}\omega_k n + \sigma^2} + \frac{b_j^\top \breve{\Sigma}_nb_i}{\sigma^2_{0, j}\omega_j n + \sigma^2} \mathbb{V}(B_{n,j}) \right]\,.
\end{align*}
\end{tiny}

The remaining terms are obtained by symmetry. Finally, for any $(i, j)$ :
\begin{align*}
&\mathbb{V}(B_{n,i}) = \mathbb{V}\left(\sum_{t \in \mathcal{A}_i}Y_{t}\right) = \omega_i n \sigma^2 + \omega_i^2 n^2 (\sigma^2_{0, i} + b_i^\top \Sigma b_i)\\
&\cov(B_{n,i}, B_{n,j}) = \cov\left(\sum_{t\in \mathcal{A}_i}Y_{t}, \sum_{t\in \mathcal{A}_j}Y_{t}\right) = \sum_{t\in \mathcal{A}_i}\sum_{t\in \mathcal{A}_j} \cov(Y_{t}, Y_{t}) = \omega_i \omega_j n^2 b_i^\top \Sigma b_j\,.
\end{align*}

\begin{remark}[Computing the upper bound for hierarchical bandit with \cref{th:linear_model_upper_bound}]
The reader can wonder why transforming the hierarchical model into a linear model thanks to \eqref{eq:first_lb}, and plug directly the transformed prior and actions to the linear upper bound (\cref{th:linear_model_upper_bound}). While this is what we do to optimize numerically the bound, it is challenging to give explicit terms with this method. In fact, it would yield to the following upper bound,
\begin{align*}
\mathcal{P}_n &\leq \doublesum \frac{1}{\sqrt{1 + \frac{ \normw{\bb_i - \bb_j}{\Bar{\Cova}_n}^2 }{ \normw{\bb_i - \bb_j}{\bSigma_n}^2 }}}e^{- \frac{ (\bnu^\top \bb_i - \bnu^\top \bb_j)^2 }{2( \normw{\bb_i - \bb_j}{\bSigma}^2 }} = \doublesum \frac{1}{\sqrt{1 + \frac{ \normw{\bb_i - \bb_j}{\Bar{\Cova}_n}^2}{\normw{\bb_i - \bb_j}{\bSigma_n}^2 }}}e^{- \frac{(\nu^\top b_i - \nu^\top b_j)^2}{2( \normw{b_i - b_j}{\Sigma}^2 + \sigma^2_{0, i}+\sigma^2_{0, j})}}\,,
\end{align*}
where
\begin{align*}
    &\bSigma_n = \left( \bSigma + \frac{1}{\sigma^2}\sum_{i \in [K]}\omega_i n \bb_i \bb_i^\top \right)^{-1}\,,\quad\Bar{\Cova}_n = \frac{1}{\sigma^4}\bSigma_n\left( \sum_{i \in [K]} \omega_i n (\sigma^2 + \bb_i \bSigma \bb_i)\bb_i \bb_i^\top + \doublesum \bb_i \bSigma \bb_j \omega_i \omega_j n^2 \bb_i \bb_j^\top \right)\bSigma_n\,.
\end{align*}
However, computing $\normw{\bb_i - \bb_j}{\Bar{\Cova}_n}^2$ and $\normw{\bb_i - \bb_j}{\bSigma_n}^2$ is computationally challenging because it requires first to compute $\bSigma$ with block-matrix inversion, then to recover the marginal and posterior covariances $\Tilde{\sigma}^2_{n, i},\,\breve{\sigma}^2_{n, i}$ and $\hat\sigma^2_{n,i}$ from \eqref{eq:cluster_posterior}, \eqref{eq:conditional_posterior} and \eqref{eq:posterior}. 
\end{remark}
\end{proof}

\subsection{Proof for Prior Misspecification}
\label{subsec:app_proof_prior_misspecification}
\begin{proof}[Proof of \cref{lemma:misspecified_prior_Gaussian}]
Let's denote as $(\tmu_{n,i}, \tsigma^2_{n,i})$ the posterior mean and variance of arm $i$ when using the misspecified priors $\cN(\tmu_{0,i}, \tsigma^2_{0})$, that is,
\begin{align*}
&\tsigma^{-2}_{n,i} = \tsigma^{-2}_{0} + n_i\sigma^{-2}\,,&\tmu_{n, i} = \tsigma^2_{n,i}( \tsigma_{0}^{-2}\mu_{0, i} +\sigma^{-2}B_{n, i})\,,
\end{align*}
where we recall $B_{n,i} = \sum_{t\in\mathcal{T}_i}Y_t$, where $\mathcal{T}_i= \{t\in[n]\,:A_t = i\}$ and $|\mathcal{T}_i = n_i|$. 
The user's decision $J_n$ is defined as $J_n = \argmax_{i\in[K]}\tmu_{n,i}$.
Then, for a fixed bandit $\theta\in\real^K$, 
\begin{align*}
    \mathbb{P}(J_n \neq i_*(\theta)\mid\theta) &\leq \mathbb{P}\left(\exists j \neq i_*(\theta)\,,\tmu_{n,j}>\tmu_{n,i_*(\theta)}\mid \theta\right)\\
    &\leq \sum_{j\in[K]}\mathbb{P}\left(\tmu_{n,j}>\tmu_{n,i_*(\theta)}\mid \theta\right)\,.
\end{align*}
Then, using the misspecified posterior means formulas and the fact that we perform uniform allocations,
\begin{align*}
\mathbb{P}\left(\tmu_{n,j}>\tmu_{n,i_*(\theta)}\mid \theta\right) &= \mathbb{P}\left(\sum_{t\in\mathcal{T}_i} Y_t - \sum_{t\in\mathcal{T}_{i_*(\theta)}} Y_t > \frac{\tsigma^2_{0,i}}{\sigma^2}\mu_{0, i_*(\theta)} - \frac{\tsigma^2_{0, i}}{\sigma^2}\mu_{0, i}\mid\theta \right)\,.
\end{align*}
Using Lemma 3 from \citet{atsidakou2022bayesian}, we can upper bound this latter probability as
\begin{small}
\begin{align*}
\mathbb{P}\left(\tmu_{n,j}>\tmu_{n,i_*(\theta)}\mid \theta\right)\leq \exp\left( -\frac{1}{4\sigma^2_0}\left(\frac{\tsigma^2_{0, i}}{\sigma^2}\mu_{0, i} - \frac{\tsigma^2_{0, i_*(\theta)}}{\sigma^2}\mu_{0, i_*(\theta)} \right)^2 - \frac{1}{2\sigma^2_0}\left(\frac{\tsigma^2_{0, i}}{\sigma^2}\mu_{0, i} - \frac{\tsigma^2_{0, i_*(\theta)}}{\sigma^2}\mu_{0, i_*(\theta)}\right) \left(\theta_i -\theta_{i^*(\theta)}\right)\right)\,.
\end{align*}
\end{small}
Considering all the possible $i_*(\theta)\in[K]$ and integrating with respect to bandit instances $\theta_i \sim \cN(\mu_{0, i}, \sigma^2_0)$ with \cref{lemma:integration_gaussian} gives the result.
\end{proof}

\section{ADDITIONAL EXPERIMENTS}
\label{sec:app_experiments}
In this section, we consider the two following settings (we set $\sigma=1$ for both settings):

\textbf{Random setting.} For MAB, each $\mu_{0,i}$ is sampled from $\mathcal{U}([0, 1])$ and $\sigma_{0,i}$'s are evenly spaced between $0.1$ and $0.5$. For linear bandits, each $\mu_{0,i}$ is sampled from $\mathcal{U}([0, 1])$ and $\Sigma_{0}$ is a diagonal matrix whose entries are evenly spaced between $0.1$ and $0.5$. For hierarchical bandits, $\nu_i\sim\mathcal{U}([-1, 1])$ and $\Sigma$ and $\Sigma_0$ are diagonal with entries evenly spaced in $[0.1, 0.5]$.

\textbf{Fixed setting.} For MAB, each $\mu_{0,i}$'s are evenly spaced between $0$ and $1$, and $\sigma_{0,i}$'s are evenly spaced between $0.1$ and $0.5$. For linear bandits, each $\mu_{0}$ is set flat, $\mu_0 = (1,\dots,1)$ and $\Sigma_{0}$ is a diagonal matrix whose entries are evenly spaced between $0.1$ and $0.5$. For hierarchical bandits, $\nu_i$'s are evenly spaced between $-1$ and $1$. $\Sigma$ and $\Sigma_0$ are diagonal with entries evenly spaced in $[0.1, 0.5]$.

Note that the random setting corresponds to the setting used for \cref{fig:main_standard_experiments}.

\subsection{Prior Misspecification}
\label{subsec:app_prior_misspecification}
We conduct additional synthetic experiments in the standard Gaussian model \eqref{eq:bayes_elim_model_gaussian}, where we instantiate $\AdaBAI$ with different allocations (uniform, optimized, G-optimal and warm-up allocations with TS) with misspecified prior. We set all the parameters as in \cref{sec:experiments}, that is, $\sigma=1$, $K=10$, $\sigma_{0, i}$ evenly spaced between $0.1$ and $0.5$ and $\mu_{0,i}\sim\mathcal{U}([0,1])$. Then construct the misspecified prior mean as $\tmu_{0} = \mu_{0} + \alpha u$, where $u\sim\cN(0, I_K)$. For the variance misspecification, we set the true prior variance to $\sigma_{0,i}=0.3$ for all arms, then construct the misspecified prior as $\tsigma_{0,i} = \sigma_{0,i}+\alpha$. Results are shown in \cref{fig:misspecification} (first row for prior misspecification, second row for variance misspecification), where we let $\alpha$ vary from $0$ (no misspecification) to $0.5$. Overall, it shows that the effect of prior misspecification vanishes as $n$ increases, as expected (see \cref{lemma:misspecified_prior_Gaussian}), except for the TS warm-up. This is because TS is more aggressive, that is, it does not explore suboptimal arms based on his prior belief, yielding to long-term consequences when sticking to the same allocation rule.

\begin{figure}[H]
\centering
\includegraphics[width=15cm]{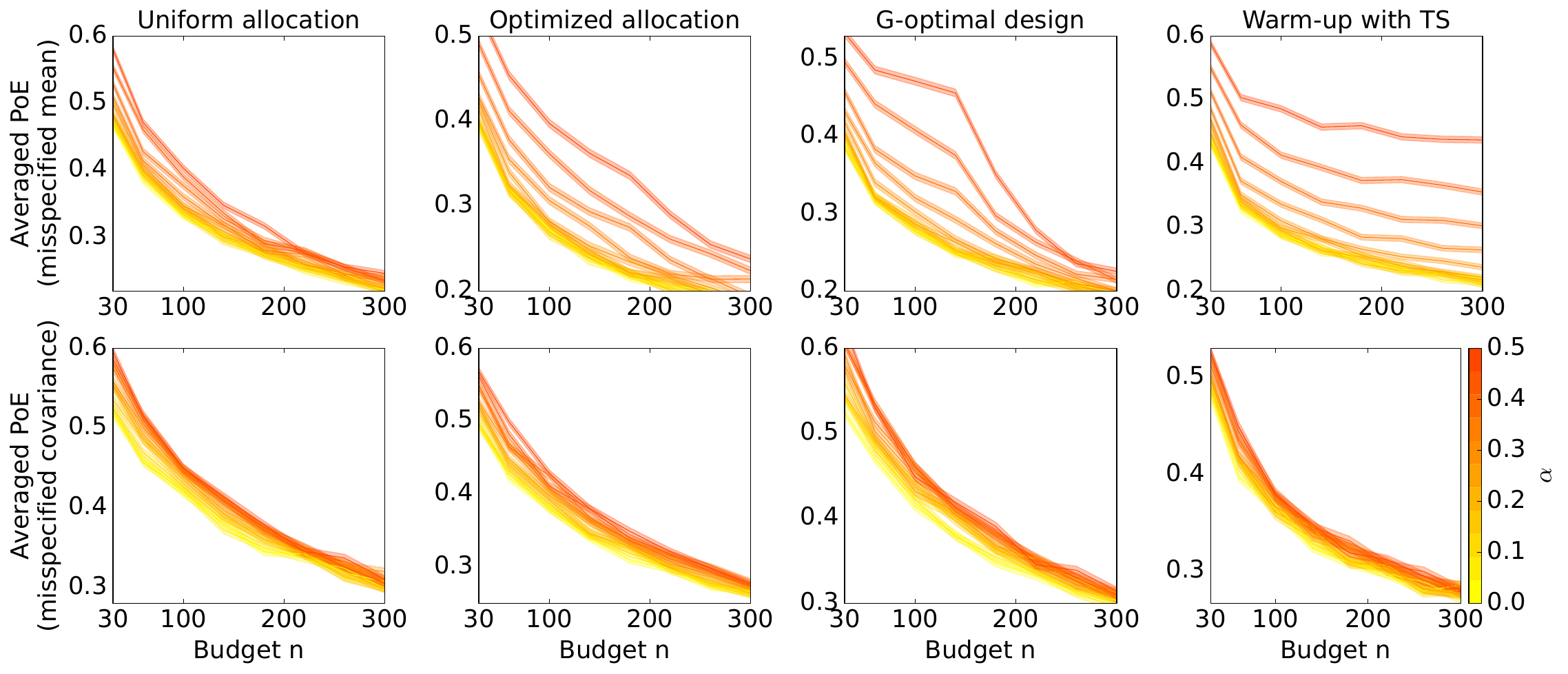}
\caption{Average PoE of $\AdaBAI$ with varying budget and allocations (uniform, optimized, G-optimal and TS warmed-up weights) for different level of mean misspecification (first row) and variance misspecification (second row).}
\label{fig:misspecification}
\end{figure}

\subsection{Additional Setting}
\label{subsec:app_additional_settings}
We provide empirical results for the Fixed setting. All remaining parameters $K$, $d$ and $L$ are the same as in \cref{fig:main_standard_experiments} of \cref{sec:experiments}. \cref{fig:fixed_setting} shows that our methods outperform existing baselines in this setting.

\begin{figure}[H]
\centering
\includegraphics[width=13cm]{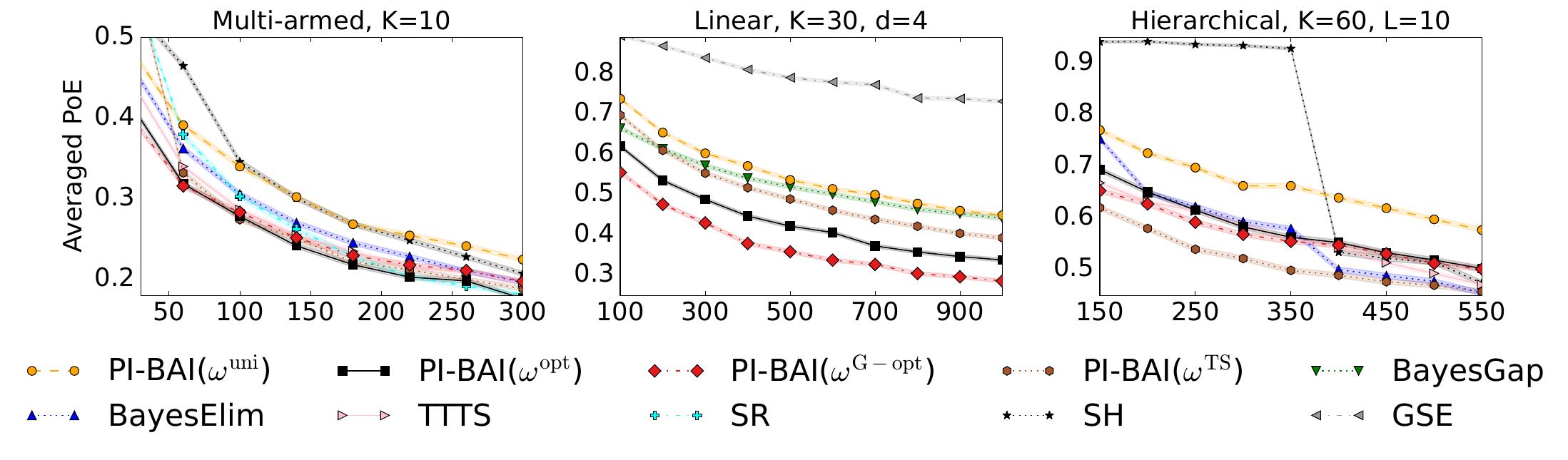}
\caption{Average PoE with varying budgets for the Fixed setting: for MAB the $\mu_{0,i}$'s are evenly spaced between $0$ and $1$ and $\sigma_{0,i}$'s are evenly spaced between $0.1$ and $0.5$. For linear bandits, each $\mu_{0,i}$ is sampled from $\mathcal{U}([0, 1])$ and $\Sigma_{0}$ is a diagonal matrix whose entries are evenly spaced between $0.1$ and $0.5$. For hierarchical bandits, $\nu_i\sim\mathcal{U}([-1, 1])$ and $\Sigma$ and $\Sigma_0$ are diagonal with entries evenly spaced in $[0.1^2, 0.5^2]$.} 
\label{fig:fixed_setting}
\end{figure}

\subsection{MovieLens Experiments}\label{subsec:app_MovieLens}
We provide more information on our MovieLens experiments in \cref{fig:main_standard_experiments}. The MovieLens dataset provides 1 million ratings ratings given by 6040 users to 3952 movies, forming a sparse rating matrix $M$ of size $6,040 \times 3,952$. To learn a prior suitable for our algorithm, we complete $M$ using alternating least squares with rank $d = 5$ (low-rank factorization), resulting in the decomposition $M = U^\top V$, where rows $U_i$ and $V_j$ represent user $i$ and movie $j$, respectively. We then construct a linear Gaussian prior (see \cref{subsec:linear_bandit}) by setting the mean $\mu_0 = \frac{1}{3952} \sum_{i \in [3952]} V_j $ and the covariance $\Sigma_0= \mathrm{diag}(v) $ where $v \in \mathbb{R}^d$ is the empirical variance of $V_j$ for $j \in [3952]$ along each dimension. We use a subset of $K=100$ randomly picked movies. In our experiments, the bandit instances aren't sampled from a Gaussian prior, but we learned a Gaussian prior that we employ in our algorithm. Despite this mismatch, $\AdaBAI$ with such prior performs very well. Of course, the prior is very informative as we used the whole dataset to learn it. But still, the bandit instances are not sampled from it. All results are averaged over $10^4$ rounds.

\subsection{Logistic Bandits}
\label{subsec:experiments_GLB}
We consider the Random and Fixed setting for $K=30$ arms and $d \in \{3, 4\}$. As explained in \cref{sec:app_discussions}, We use a Laplace approximation to estimate the posterior means. For the warmed-up allocations, we use TS with Bernoulli-logistic likelihood \citep{chapelle11empirical}. The G-optimal design weights are obtained as if we face a linear Gaussian bandit.

\cref{fig:GLB} shows that the generalization of \AdaBAI with G-optimal design allocations on has good performances beyond linear settings.
\begin{figure}[H]
\centering
\includegraphics[width=15cm]{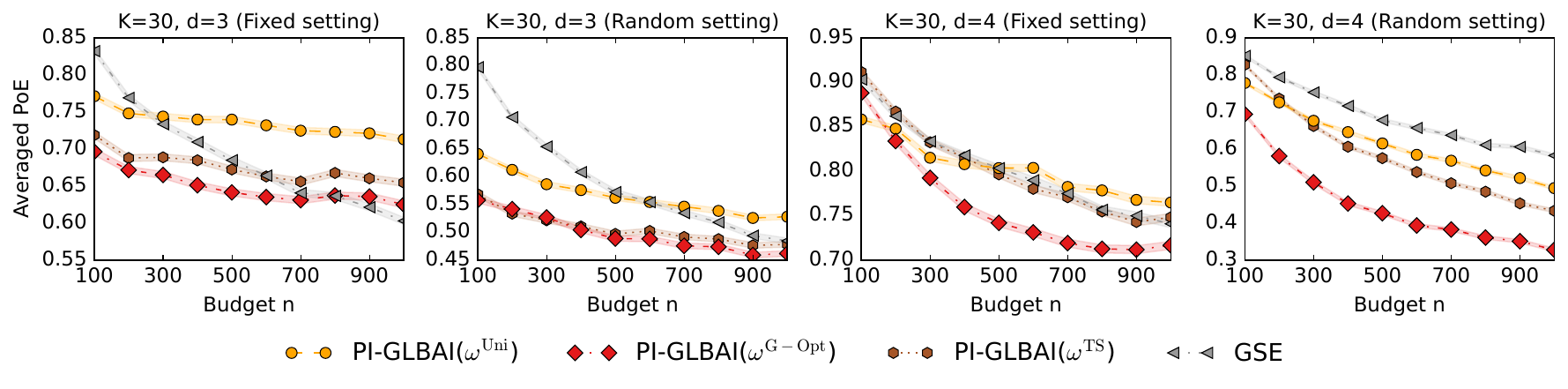}
\caption{Average PoE with varying budgets for the Fixed setting and the Random setting in the GLB framework.} 
\label{fig:GLB}
\end{figure}
\subsection{Choice of Baselines}
\label{subsec:choice_baselines}

\textbf{A remark on \texttt{TTTS}.}
Top two sampling algorithms is a family of algorithms that is known to have good performances in BAI. In \cref{sec:experiments}, we used \texttt{TS-TCI} with $\beta=0.5$ from \citet{jourdan2022top} and denoted it as \texttt{TTTS} for sake of notation simplicity.

\textbf{Influence of elimination.} We empirically compare the influence of using elimination on top of our methods. The elimination procedure is the same as the one used in \citet{atsidakou2022bayesian}. There are $\lfloor \log_2(K) \rfloor$ rounds, and each lasts $\lfloor \frac{n}{R} \rfloor$ steps. At each round, we pull each remaining arm $i$ $\lfloor \frac{\omega_i n}{R} \rfloor$ times. At the end of the round, half of arms are eliminated. These correspond to the arms that have the least posterior mean reward (so $\hat\mu_{n,i}$ in the MAB setting). The allocation $\omega$ is then normalized to allocate more budget to remaining arms. Note that we draft all observations at the end of each round, as it is the case in all methods that rely on successive halvings \citep{atsidakou2022bayesian,azizi2021fixed,karnin2013almost}. \cref{fig:elimination} shows that using elimination does not give better performances, and hence we chose to not add these baselines in \cref{sec:experiments}.

\begin{figure}[H]
\centering
\includegraphics[width=\linewidth]{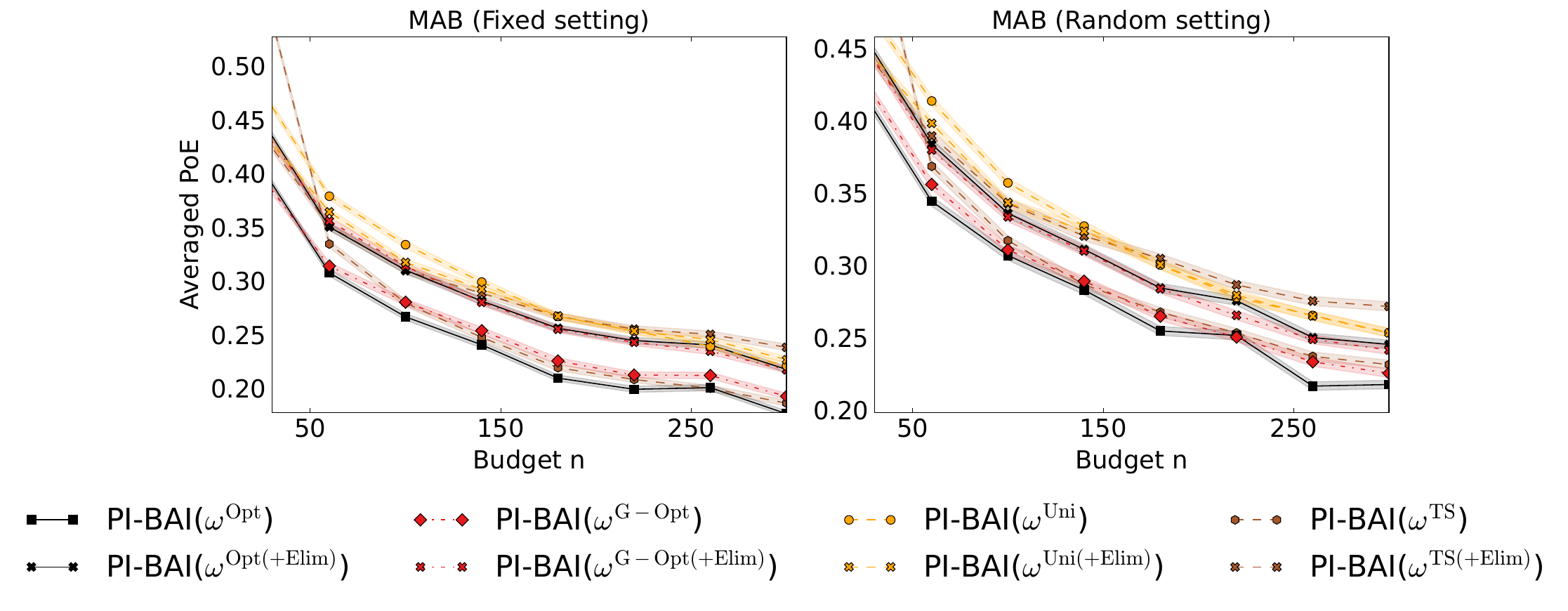}
\caption{Average PoE of \AdaBAI instantiated with different weights with or without elimination in the MAB setting.} 
\label{fig:elimination}
\end{figure}

\subsection{Simple Regret}
\label{subsec:simple_regret}
\cref{fig:simple_regret} compares the performances of our methods based on the Bayesian simple regret (see discussion in \cref{sec:app_related_work}). Overall, it shows that our method outperform existing baselines in terms of simple regret in the Random and Fixed settings. We do not compare our method to \citet{komiyama2023rate}, since it is designed for Bernoulli bandits.

\begin{figure}[H]
\centering
\includegraphics[width=15cm]{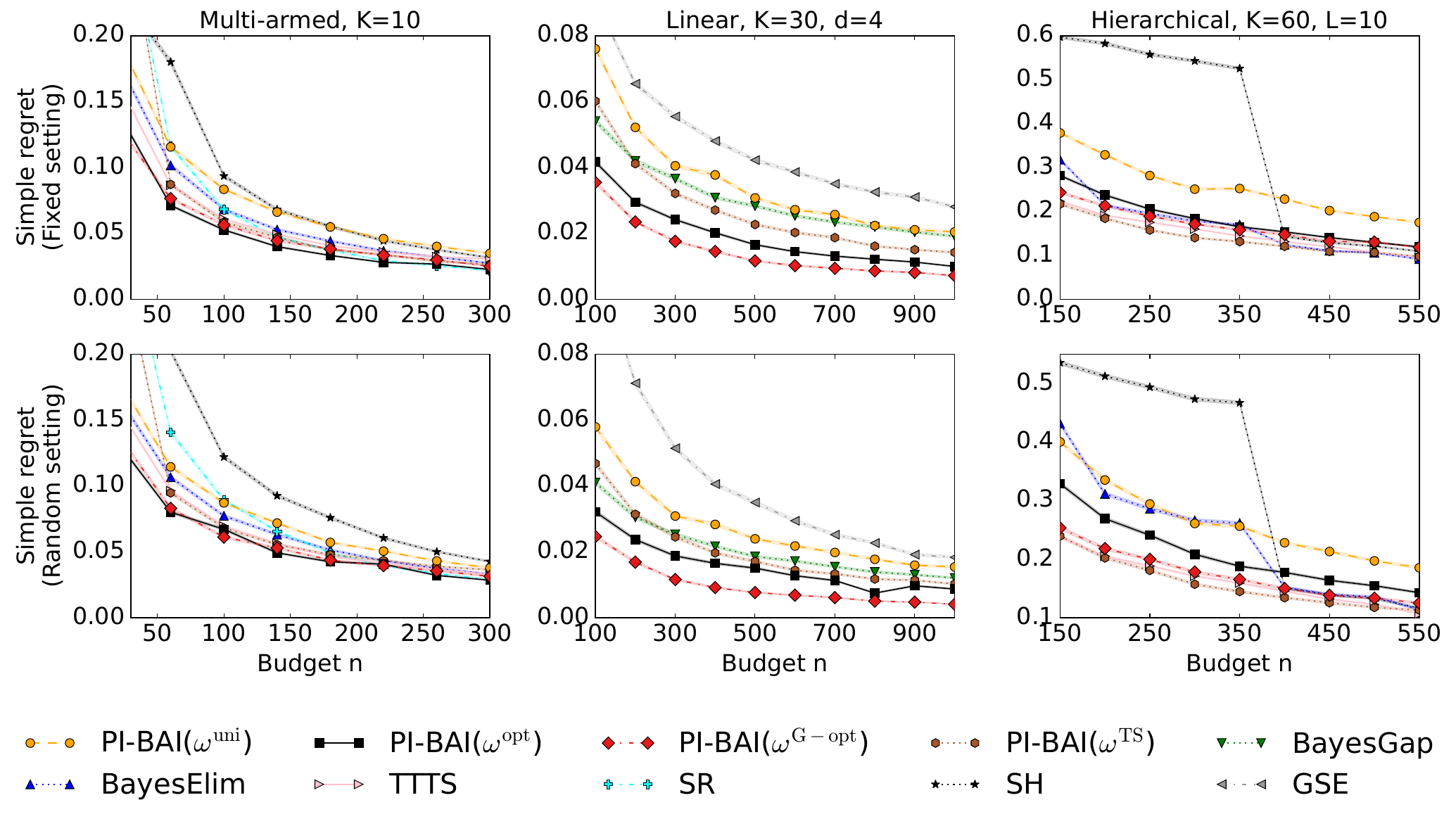}
\caption{Average simple regret with varying budgets for fixed and randomized settings.} 
\label{fig:simple_regret}
\end{figure}

\subsection{Hyperparameters}
\label{subsec:app_experiments_hyperparameters}

\textbf{Warm-up length $n_{\rm{w}}$.}
We try different values of warm-up length $n_{\rm{w}}$ for our warm-up policies. We emphasize that methods based on \texttt{TTTS} require $n_{\rm{w}} > K$ because each arm has to be pulled at the beginning. \cref{fig:hyperparameters_warm_up} suggests picking $n_{\rm{w}} = 2K$ for the warm-up with \texttt{T3C} and  \texttt{TSTCI}, and $n_{\rm{w}} = K$ for the warm-up with TS.

\begin{figure}[H]
\centering
\includegraphics[width=\linewidth]{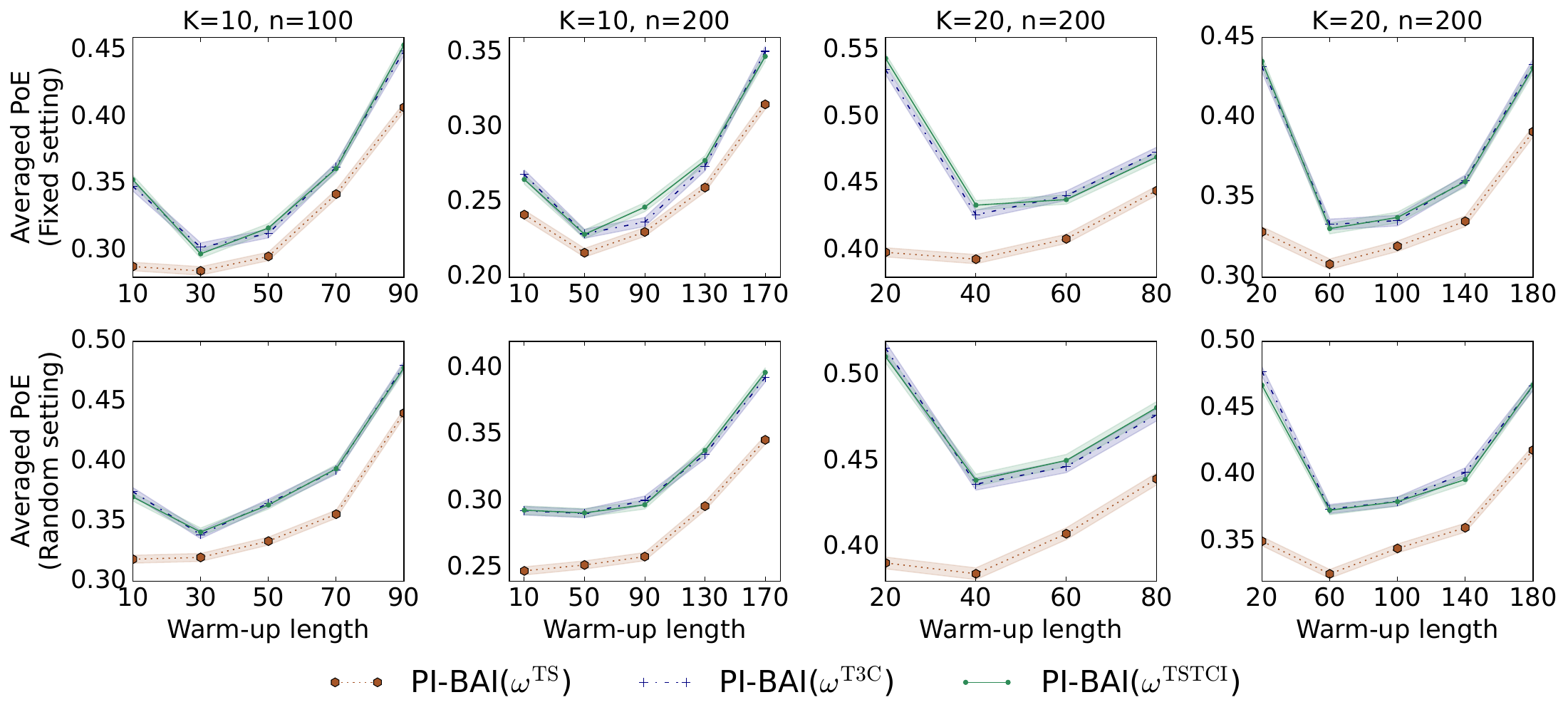}
\caption{Average PoE of \AdaBAI instantiated with different warm-up policies for different warm-up lengths $n_{\rm{w}}$.} 
\label{fig:hyperparameters_warm_up}
\end{figure}

\textbf{Choice of warm-up policy.} We evaluate different warm-up policies,\emph{TS} and two Top-Two algorithms, \texttt{TSTCI} and \texttt{T3C} from \citet{jourdan2022top}. The experiments shown in \cref{fig:warm_up_choice} are run in the same setting as in \cref{sec:experiments}, with $K=10$ arms in the MAB setting, and with $K=60$ and $d=4$ in the hierarchical setting. \cref{fig:warm_up_choice} suggests to pick \texttt{TS} as a warm-up policy for the MAB setting and \texttt{meTS} of \citet{aouali2022mixed} for the hierarchical setting. 
\begin{figure}[H]
\centering
\includegraphics[width=\linewidth]{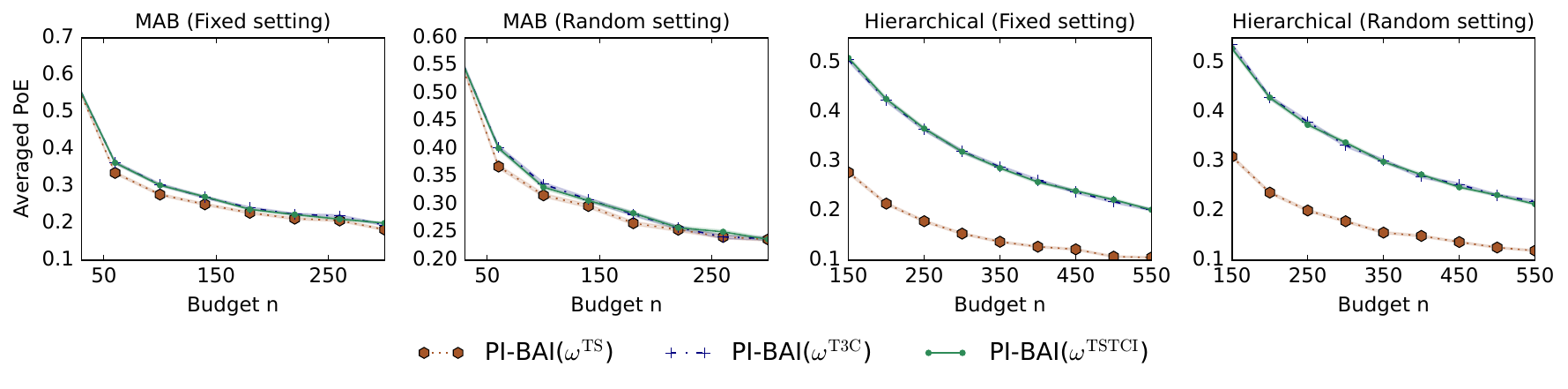}
\caption{Average PoE of \AdaBAI instantiated with different warm-up policies.} 
\label{fig:warm_up_choice}
\end{figure}

\textbf{Mixture parameter $\alpha$.} We discuss the choice of the mixture parameter $\alpha$ in the settings presented in \cref{sec:experiments,subsec:app_additional_settings}. We recall that in the "Random setting", each $\mu_{0,i}$ is sampled from $\mathcal{U}([0, 1])$, whereas in the "fixed setting", $\mu_{0,i}$ are evenly spaced between $0$ and $1$. In both settings, $\sigma_{0,i}$ are evenly spaced from $0.1$ to $0.5$. We recall that  we use the heuristic $\alpha \wOpt_i + (1 - \alpha)\frac{\mu_{0,i}\sigma_{0,i}}{\sum_{k \in [K]}\mu_{0,k}\sigma_{0,k}}$ in our experiments. \cref{fig:hyperparameters_alpha} shows that for the fixed setting, adding the vector $\frac{\mu_{0,i}\sigma_{0,i}}{\sum_{k \in [K]}\mu_{0,k}\sigma_{0,k}}$ helps improve the performances. This is not necessarily the case in the random setting. We find that the best performance is reached at $\alpha\approx 0.5$.
\begin{figure}[H]
\centering
\includegraphics[width=\linewidth]{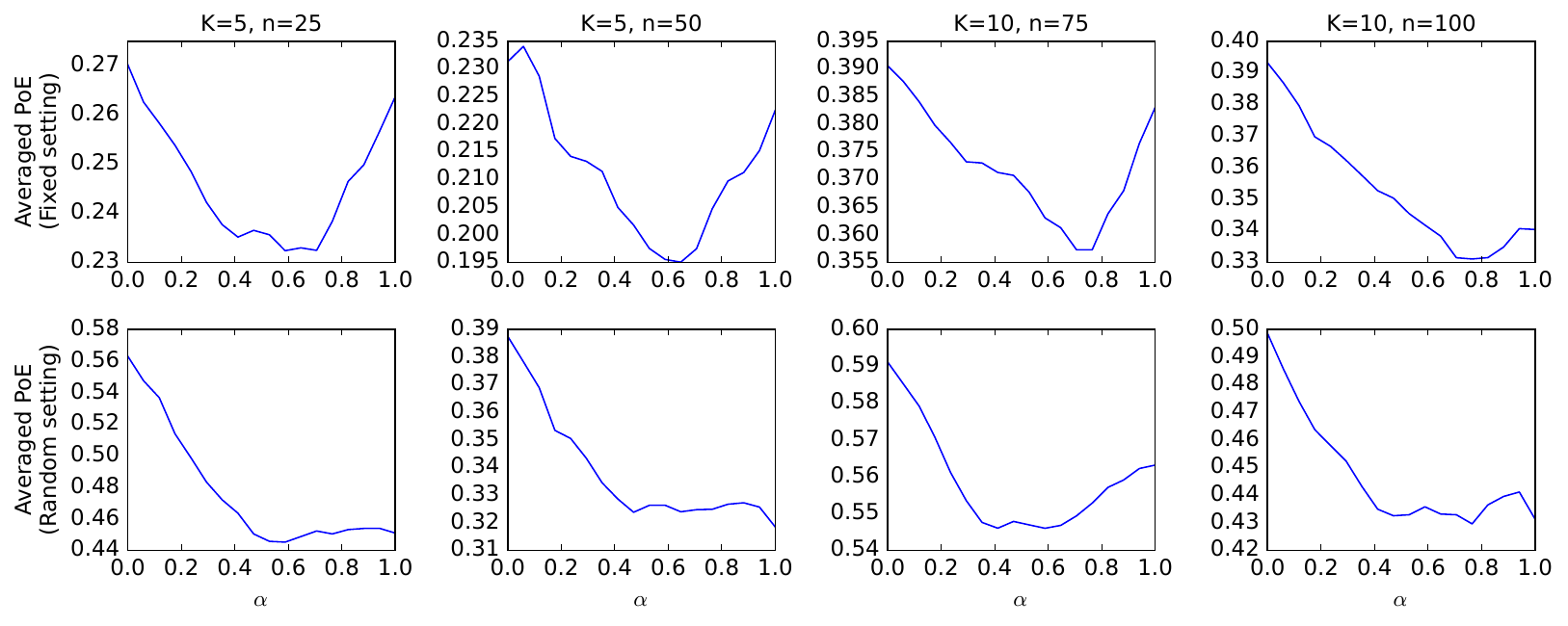}
\caption{Average PoE of $\AdaBAI(\wOpt)$ for different mixture parameter $\alpha$.} 
\label{fig:hyperparameters_alpha}
\end{figure}

\subsection{Confidence Intervals on Sampled Instances}
\label{app:subsec_confidence_intervals}
We show different type of confidence intervals in the Random and Fixed settings. In the first row of \cref{fig:confidence_intervals}, for each instance, we repeat the experiments 100 times on a same bandit environment to get a probability of error. Then, we sample 1000 different bandit instances and show one standard deviation of the resulting probability of error. In the second row of \cref{fig:confidence_intervals}, we plot the PoE of each method subtracted by the PoE of the method having the least PoE in each setting ($\AdaBAI(\wOpt)$ in MAB, $\AdaBAI(\wGOpt)$ in linear bandits and $\AdaBAI(\wTS)$ in hierarchical bandits).
\begin{figure}[H]
\centering
\includegraphics[width=16cm]{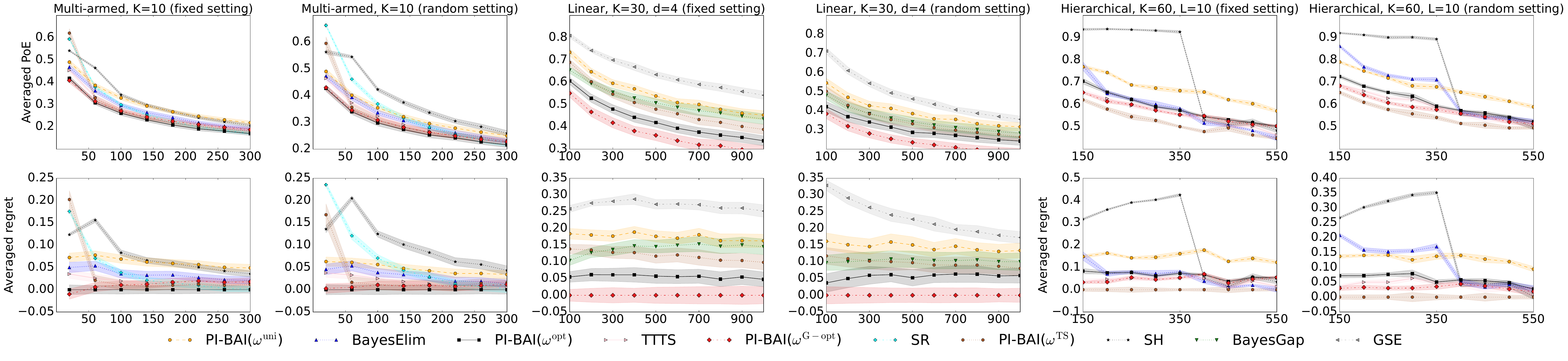}
\caption{PoE of several methods (first row) and PoE of each method substracted by the PoE of the most performing method in each setting (second row). For each instance, we repeat the experiments 100 times and we average the results over 1000 instances. The confidence intervals show one standard deviation.}
\label{fig:confidence_intervals}
\end{figure}

\section{COMPUTE RESSOURCES}
\label{sec:app_compute_ressources}
All experiments have been run locally on a 2021 Mac Book Pro with 8 CPUs.

\end{document}